%% file: main.tex
\documentclass{article}

\usepackage{enumitem}
\usepackage{bbm}
\usepackage{multirow}
\input{math_commands.tex}

\usepackage{amsmath,amssymb} 
\usepackage{url}
\usepackage{amsthm}
\usepackage{thmtools} 
\usepackage{symbols}
\usepackage{makecell}
\usepackage{graphicx}
\usepackage{booktabs}
\usepackage[ruled,vlined,linesnumbered]{algorithm2e} 
\usepackage{xcolor}

\usepackage[colorlinks=true]{hyperref}%
\usepackage[accepted]{icml2021}

\icmltitlerunning{Cooperative Multi-Agent Exploration}

\begin{document}
\twocolumn[
\icmltitle{Cooperative Exploration for Multi-Agent Deep Reinforcement Learning}

\icmlsetsymbol{equal}{*}

\begin{icmlauthorlist}
\icmlauthor{Iou-Jen Liu}{uiuc}
\icmlauthor{Unnat Jain }{uiuc}
\icmlauthor{Raymond A. Yeh}{uiuc}
\icmlauthor{Alexander G. Schwing}{uiuc}
\end{icmlauthorlist}

\icmlaffiliation{uiuc}{University of Illinois at Urbana-Champaign, IL, U.S.A.}

\icmlcorrespondingauthor{Iou-Jen Liu}{iliu3@illinois.edu}

\icmlkeywords{Machine Learning, ICML}

\vskip 0.3in
]
\printAffiliationsAndNotice{}
\input{abs}

\input{intro}

\input{back}

\input{app}

\input{exp}

\input{rel}

\input{conc}

\bibliography{reference}
\bibliographystyle{icml2021}
\clearpage
\appendix
\section{Appendix}

\input{supp}

\end{document}

%% file: math_commands.tex
\usepackage{amsmath,amsfonts,bm}

\def\1{\bm{1}}

\def\vs{{\bm{s}}}

\DeclareMathAlphabet{\mathsfit}{\encodingdefault}{\sfdefault}{m}{sl}
\SetMathAlphabet{\mathsfit}{bold}{\encodingdefault}{\sfdefault}{bx}{n}

\def\gO{{\mathcal{O}}}

\DeclareMathOperator*{\argmin}{arg\,min}

%% file: abs.tex
\begin{abstract}
Exploration is critical for good results in deep reinforcement learning 
and has attracted much attention. However, existing \emph{multi-agent} deep reinforcement learning algorithms still use mostly noise-based techniques. 
Very recently, exploration methods that consider cooperation among multiple agents have been developed. However, existing methods suffer from a common challenge: agents struggle to identify states that are worth exploring, and hardly coordinate  exploration efforts toward those states. To address this shortcoming, in this paper, we propose cooperative multi-agent exploration (CMAE): agents share a common goal while exploring. The goal is selected from multiple projected state spaces via a normalized entropy-based technique. Then, agents are trained to reach this goal in a coordinated manner. 
 We demonstrate that CMAE consistently outperforms baselines on various tasks, including a sparse-reward version of the  multiple-particle environment (MPE) and the  Starcraft multi-agent challenge (SMAC).   
 \end{abstract}

%% file: intro.tex
\section{Introduction}
\label{sec:intro}

Multi-agent reinforcement learning (MARL) is an increasingly important field. Indeed, many real-world problems are naturally modeled using MARL techniques. 
For instance, tasks from areas as diverse as robot fleet coordination~\citep{Swamy20, swarm} and autonomous traffic control~\citep{Bazzan08, Sunehag18} fit MARL formulations.

To address MARL problems, early work followed  independent single-agent reinforcement learning work~\citep{Tampuu15, Tan93, Matignon12}. However, more recently, specifically tailored techniques such as monotonic value function factorization (QMIX)~\citep{qmix}, multi-agent deep deterministic policy gradient (MADDPG)~\citep{maddpg}, and counterfactual multi-agent policy gradients (COMA)~\citep{coma} have been developed. Those methods excel in a multi-agent setting because they address the non-stationary issue of MARL {via a centralized critic.} 
Despite those advances and the resulting reported performance improvements, a common issue remains: all of the aforementioned methods use exploration techniques from classical algorithms. Specifically, these methods employ noise-based exploration, \ie, the exploration policy is a noisy version of the actor policy~\cite{wqmix, maddpg, Foerster16, qmix, Yang18}.

It was recently recognized  that use of classical exploration techniques is sub-optimal in a multi-agent reinforcement learning setting. Specifically,~\citet{maven} show that QMIX with $\epsilon$-greedy exploration results in slow exploration and sub-optimality.~\citet{maven} improve exploration by conditioning an agent's behavior on a shared latent variable controlled by a hierarchical policy. 
Even more recently,~\citet{eiti} encourage coordinated exploration by considering the influence of one agent's behavior on other agents' behaviors.

While all of the aforementioned  exploration techniques  for multi-agent reinforcement learning significantly improve results, 
{they suffer from two common challenges: 
(1) agents struggle to identify states that are worth exploring. .
Identifying under-explored states is particularly challenging when the number of agents increase, since the state and action space grows exponentially with the number of agents. 
(2) Agents don't coordinate their exploration efforts toward under-explored states. }
To give an example, consider a Push-Box task, where two agents need to jointly push a heavy box  to a specific location before observing a reward.  In this situation, instead of exploring  the environment independently, agents need to coordinate pushing the box within the environment to find the specific location. 

To address both challenges, we propose cooperative multi-agent exploration (CMAE). 
{To identify states that are worth exploring, 
we observe that, while the state space grows exponentially, the reward function typically depends on a small subset of the state space. For instance, in the aforementioned Push-Box task, the state space contains the location of agents and the box while the reward function only depends on the location of the box. To solve the task, exploring the box's location is much more efficient than exploring the full state space. To encode this inductive bias into CMAE, we propose a bottom-up exploration scheme. Specifically, we project the high-dimensional state space to low-dimensional spaces, which we refer to as restricted spaces. 
Then, we gradually explore restricted spaces from low- to high-dimensional. 
To ensure the agents coordinate their exploration efforts, we select goals from restricted spaces and train the exploration policies to reach the goal. Specifically, inspired by~\citet{her}, we reshape the rewards in the replay buffer such that a positive reward is given when the goal is reached. }

{To show that CMAE improves results, we evaluate the proposed approach on two multi-agent environment suites: a discrete version of the multiple-particle environment (MPE)~\cite{maddpg, eiti} and the Starcraft multi-agent challenge (SMAC)~\citep{smac}. In both environments, we consider both dense-reward and sparse-reward settings. Sparse-reward settings are particularly challenging because agents need to coordinate their behavior for extended timesteps before receiving any non-zero reward. CMAE consistently outperforms the state-of-the-art baselines in sparse-reward tasks. {For more, please see our project page: \url{https://ioujenliu.github.io/CMAE}}. }

%% file: back.tex
\section{Preliminaries}
We first define  the multi-agent Markov decision process (MDP) in~\secref{subsec:mdp} and introduce the multi-agent reinforcement learning setting in~\secref{subsec:marl}.
\subsection{Multi-Agent Markov Decision Process}
\label{subsec:mdp} 
A cooperative multi-agent system is modeled as a multi-agent Markov decision process (MDP). An $n$-agent MDP is defined by a tuple $({\cal S}, {\cal A}, {\cal T}, {\cal R}, {\cal Z}, {\cal O}, n, \gamma, H)$. $\cal S$ is the state space of the environment.  
 $\cal A$ is the action space of each agent. At each time step $t$, each agent's target policy $\pi_i$,  $i \in \{1,\ldots,n\}$, selects an action $a_i^t \in {\cal A}$. All selected actions form a joint action $\bm{a}^t \in {\cal A}^n$. The transition function ${\cal T}$ maps the current state $s^t$ and the joint action $\bm{a}^t$ to a distribution over the next state $s^{t+1}$, \ie, ${\cal T}: {\cal S} \times {\cal A}^n \rightarrow \Delta({\cal S})$. All agents receive a collective reward $r^t \in \mathbb{R}$ according to the reward function ${\cal R}: {\cal S} \times {\cal A}^n \rightarrow \mathbb{R}$. The
objective
 of all agents' policies is to maximize the collective return $\sum_{t=0}^H \gamma^t r^t$, where $\gamma \in [0, 1]$ is the discount factor, $H$ is the horizon,  and $r^t$ is the collective reward obtained at timestep $t$. Each agent $i$ observes local observation $o^t_i \in {\cal Z}$ according to the observation function ${\cal O}:  {\cal S} \rightarrow {\cal Z}$. 
Note,  observations usually reveal partial information about the state. For instance, suppose the state contains the location of agents, while the local observation of an agent may only contain the location of other agents within a limited distance.  All agents' local observations form a joint observation $\bm o^t$. 
\vspace{-0.25cm}
\subsection{Multi-Agent Reinforcement Learning}
\label{subsec:marl} 
In this paper, we follow the standard centralized training and decentralized execution (CTDE) paradigm~\citep{maddpg, qmix, coma, maven, pic}: at training time, the learning algorithm has access to all agents' local observations, actions, and the state. At execution time, \ie, at test time, each individual agent only has access to its own local observation.   

The proposed CMAE is applicable to off-policy MARL methods \citep[\eg,][]{qmix, maddpg, Sunehag18, Matignon12, pic}. In off-policy MARL, exploration policies $\bm{\mu}=\{\mu_i\}_{i=1}^n$ are responsible for collecting data from the environment. The data  in the form of transition tuples $(s^t, \bm{o}^t, \bm{a}^t, s^{t+1} , \bm{o}^{t+1}, r^t)$ is stored in  a replay memory $\cal D$, \ie, ${\cal D}=\{ (s^t, \bm{o}^t, \bm{a}^t, s^{t+1} , \bm{o}^{t+1}, r^t) \}_t$. The target policies $\bm{\pi}=\{\pi_i\}_{i=1}^n$ are trained using transition tuples from the replay memory.

%% file: app.tex

\vspace{-0.2cm}
\section{Coordinated Multi-Agent Exploration (CMAE)}
\label{sec:app}
In the following we first present an overview of CMAE before we discuss the method more formally.

\SetKwComment{Comment}{$\triangleright$\ }{}

\setlength{\textfloatsep}{5pt}
\begin{algorithm*}[tb]
\DontPrintSemicolon

\SetKwInOut{Input}{input}
 \SetKwInOut{Output}{output}
\SetKwInput{init}{Init}

\init{space tree $ T_{\text{space}}$, counters $\bm{c}$ } 
\init{exploration policies $\bm{\mu}=\{\mu_i\}_{i=1}^n$, target policies $\bm{\pi}=\{\pi_i\}_{i=1}^n$,  replay buffer $\cal D$}

\For{episode $=1\dots E$}{
  Reset the environment. Observe state $s^1$ and observations $\bm{o}^1 =( o_1^1,\dots, o_n^1)$\\
 \For{$t=1 \dots H$}{
  Select $\bm{a}^t$ \text{using a mixture of exploration and target policies} $\alpha \bm{\mu}+ (1-\alpha)\bm{\pi}$. $\alpha$ decreases linearly to 0 \label{alg1:select}\\
  $r^t$, $s^{t+1}$, $\bm{o}^{t+1}$\ = environment.step($\bm{a}^t$) \\
  Add transition tuple $\{s^t, \bm{o}^t, \bm{a}^t, s^{t+1}, \bm{o}^{t+1}, r^t\}$ to $\cal D$\\
  UpdateCounters($\bm{c}$, $s^{t+1}$, $\bm{o}^{t+1}$)\label{alg1:track}\\ 
  TrainTarget($\bm{\pi}$, ${\cal D}$)\\
 }

 \If{episode $\bmod~N = 0$}{
$g$ = SelectRestrictedSpaceGoal($\bm{c}$, $T_{\text{space}}$, $\cal D$, episode)\label{alg1:space}\Comment*[r]{Select shared goal (\algref{alg:select_goal})}
 }
   TrainExp($\bm{\mu}$, $g$, ${\cal D}$)\Comment*[r]{Train exploration policies (\algref{alg:trainexp})}
}
\caption{Training with Coordinated Multi-Agent Exploration (CMAE)}
\label{alg:cmae}
\end{algorithm*}

\setlength{\textfloatsep}{0pt}

\textbf{Overview:} 
The 
goal 
is to train the target policies $\bm{\pi}=\{\pi_i\}_{i=1}^n$ of $n$ agents to maximize the environment episode return. Classical off-policy algorithms~\citep{maddpg, qmix, dqn1, dqn2, ddpg} typically use a noisy version of the target policies $\bm{\pi}$ as  exploration policies.   
In contrast, in CMAE, we decouple exploration policies and target policies.  
Specifically, target polices are trained to maximize the usual external episode return. 
{
Exploration policies $\bm{\mu} = \{\mu_i\}_{i=1}^n$ are trained to reach shared goals, which are under-explored states, as 
the job of an exploration policy  is to collect data from those under-explored states.

To train the exploration policies, shared goals are required. How to choose shared goals from a high-dimensional state space? As discussed in~\secref{sec:intro}, while the state space grows exponentially with the number of agents, the reward function often only depends on a small subset of the state space. Concretely, consider an $n$-agent Push-Box game in a $L \times L$ grid. The size of its state space is $L^{2(n + 1)}$ ($n$ agents plus box in $L^2$ space).
However, the reward function depends only on the location of the box, whose  state space size is $L^2$.
Obviously, to solve the task, exploring the location of the box is much more efficient than uniformly exploring the full state space. 

To achieve this, CMAE first explores a low-dimensional restricted space $\cS_k$ of the state space $\cS$, \ie, $\cS_k \subseteq \cS$. {Formally, given an $M$-dimensional state space $\cal S$, 
the restricted space ${\cal S}_k$ associated with a set $k$ is defined as 
\begin{equation}
    {\cal S}_k = \{\text{proj}_k(s): \forall s \in \cal S\},
    \label{eq:def_restricted}
\end{equation}
where $\text{proj}_k(s) = (s_e)_{e \in k}$ `restricts' the space to elements $e$ in set $k$, \ie, $e\in k$.  Here, $s_e$ is the $e$-th component of the full state $s$, and $k$ is a set from the power set of $\{1,\dots, M\}, \ie,~k \in P(\{1,\dots, M\})$, where $P$ denotes the power set.} 
CMAE gradually moves from low-dimensional restricted spaces ($|k|$ small) to higher-dimensional restricted spaces ($|k|$ larger). This bottom-up space selection is formulated as a search  on a space tree $T_{\text{space}}$, where each node represents a  restricted space $\cS_k$. 

}

\begin{algorithm}[tb]
\SetKwInOut{Input}{input}
 \SetKwInOut{Output}{output}
 \SetKwInput{init}{Init}
\DontPrintSemicolon
\Input{exploration policies $\bm{\mu}=\{\mu_i\}_{i=1}^n$, shared goal $g$, replay buffer $\cal D$}
\For{$\{s^t, \bm{o}^t, \bm{a}^t, s^{t+1}, \bm{o}^{t+1}, r^t\} \in \cal D$}{
    \If{$s^t$ is the shared goal $g$}{
        $r^t = r^t + \hat{r}$\\
    }
    Update $\bm{\mu}$ using $\{s^t, \bm{o}^t, \bm{a}^t, s^{t+1}, \bm{o}^{t+1}, r^t\}$
}
 \caption{Train Exploration Policies (TrainExp)}
 \label{alg:trainexp}
\end{algorithm}

\algref{alg:cmae} summarizes this approach. At each step,  a mixture of the exploration policies $\bm{\mu} = \{\mu_i\}_{i=1}^n$ and target policies $\bm{\pi} =\{\pi_i\}_{i=1}^n$ is used to select actions (line~\ref{alg1:select}). {The resulting experience tuple is then {stored} in a replay buffer $\cal D$. Counters {$\bm{c}$}  
for each restricted space {in the space tree $T_{\text{space}}$} track how often a particular restricted state was observed (line~\ref{alg1:track}). 
The target policies $\bm{\pi}$ are trained directly using the data within the replay buffer $\cal D$. 
Every $N$ episodes, a new restricted space and goal {$g$} is chosen (line~\ref{alg1:space}; see \secref{subsec:selection} for more). Exploration policies are continuously trained to reach the selected goal (\secref{subsec:train_exp})}.

\subsection{Training of Exploration Policies}
\label{subsec:train_exp}

To encourage that exploration policies scout environments in a coordinated manner, we train the exploration policies $\bm{\mu}$ with an additional modified reward $\hat{r}$. This modified reward emphasizes the goal $g$ of the exploration. For example, in the two-agent Push-Box task, we  use a particular joint location of both agents and the box as a goal. 
Note, the agents receive a bonus reward $\hat{r}$ when the shared goal $g$, \ie, the specified state, is reached. The algorithm for training exploration policies is summarized in~\algref{alg:trainexp}: standard policy training with a modified reward.

The goal $g$ is obtained 
via a bottom-up search method. We first explore low-dimensional restricted spaces that are `under-explored.'  We discuss this shared goal and restricted space selection method next. 

\begin{algorithm}[tb]
\DontPrintSemicolon
\SetKwInOut{Input}{input}
 \SetKwInOut{Output}{output}
\SetAlgoLined
\Input{counters $\bm{c}$, space tree $T_{\text{space}}$, replay buffer $\cal D$, episode}
\Output{selected goal $g$}
Compute utility of restricted spaces in $T_{\text{space}}$\\
Sample a restricted space ${\cal S}_{k^\ast}$ from $T_{\text{space}}$ following~\equref{eq:sample}\\
Sample a batch $B = \{s_i\}_{i=1}^{|B|}$ from $\cal D$\\
$g = \argmin_{s\in{ B}} c_{k^\ast}(\text{proj}_{k^\ast}(s))$\\
   \If {episode $\bmod~N' = 0$}{

ExpandSpaceTree($\bm{c}$, $T_{\text{space}}$, $k^\ast$)  \Comment*[r]{\secref{subsec:selection}}
}

 \Return $g$\\
 \caption{Select Restricted Space and Shared Goal (SelectRestrictedSpaceGoal)}
 \label{alg:select_goal}
\end{algorithm}
\setlength{\textfloatsep}{5pt}
\subsection{Shared Goal and Restricted Space Selection}
\label{subsec:selection}
\vspace{-0.1cm}

Since the size of the state space $\cS$ grows exponentially with the number of agents, conventional exploration strategies~\cite{maven, rmax} which strive to uniformly visit all states are no longer tractable.   
To address this issue, we propose to first project the state space to restricted spaces, and then perform shared goal driven coordinated exploration in those restricted spaces. For simplicity, we first assume the state space is finite and discrete. We discuss how to extend CMAE to continuous state spaces in~\secref{subsec:continuous_count}. 
In the following, we first show how to select the goal $g$ given a {selected restricted space}. Then we discuss how to select restricted spaces and {expand the space tree $T_{\text{space}}$.}

\noindent {\bf Shared Goal Selection:} Given a restricted space $\cS_{k^{*}}$ and its associated counter $c_{k^*}$, we choose the goal state $g$ by first uniformly sampling a batch of states $B$ from the replay buffer $\cD$. From those states,  we  select the state with the smallest count as the goal state $g$, \ie, 
\be
g = \arg\min_{s\in B} c_{k^*}({\text{proj}_{k^*}(s)}),
\ee 
where $\text{proj}_{k^*}(s)$ is the projection from state space to the restricted space $\cS_{k^*}$. 
To make this concrete, consider again the 2-agent Push-Box game. A restricted space may consist of only the box's location. Then, from batch $B$, a state in which the box is in a rarely seen location will be selected as the goal. 
{For each restricted space $\cS_{k^*}$, the associated counter $c_{k^\ast}(s_{k^\ast})$ stores the number  of times the state $s_{k^\ast}$ occurred in the low-dimensional restricted space $\cS_{k^\ast}$. }

Given a goal state $g$, we train  exploration policies $\bm{\mu}$ using the method presented in~\secref{subsec:train_exp} (\algref{alg:trainexp}).

{\bf Restricted Space Selection:} 
{For an $M$-dimensional state space, the number of restricted spaces is equivalent to the size of the power set, \ie, $2^M$. It's intractable to study all.}

To address this issue, we propose a bottom-up tree-search mechanism to select under-explored restricted spaces.  We start from low-dimensional restricted spaces and then gradually grow the search tree to explore higher-dimensional restricted spaces. 
Specifically, we maintain a space tree $T_{\text{space}}$ where each node in the tree represents a restricted space. Each restricted space $k$ in the tree is associated with a utility value $u_k$, which guides the 
selection.

The utility permits to identify the under-explored restricted spaces. 
For this, we study a normalized-entropy-based mechanism to compute the utility of each restricted space in $T_\text{space}$. 
Intuitively, under-explored restricted spaces have lower normalized entropy.  
To estimate the normalized entropy of a restricted space $\cS_k$,
we normalize the 
{counter $c_k$}
to obtain a probability distribution
{$p_k(\cdot) = c_k(\cdot)/\sum_{s\in\cS_k} c_k(s)$,}
which is then used to compute the normalized entropy 
\begin{equation}
\begin{aligned}
\eta_k &= H_k/H_{\text{max}, k} 
        = -\left(\sum_{s \in \cS_k}p_k(s)\log p_k(s)\right)/\log(|\cS_k|). 
\label{eq:ent}
\end{aligned}
\end{equation}

Then the utility $u_k$ is given by $u_k = -\eta_k$.
Finally, we sample a restricted space $k^\ast$ following {a categorical distribution over all spaces in the space tree $T_\text{space}$, \ie, 
\be
k^\ast \sim \text{Cat}(\text{softmax}((u_k)_{\cS_k\in T_{\text{space}}})).
\label{eq:sample}
\ee
}%
The restricted space and goal selection method is summarized in~\algref{alg:select_goal}.
{Note that the actual value of $|\cS_k|$ is usually unavailable. We defer the details of estimating $|\cS_k|$ from observed data to~\secref{subsec:ent} in the appendix.} 

\noindent {\bf Space Tree Expansion:}
The space tree $T_\text{space}$ is initialized with one-dimensional restricted spaces.  To consider restricted spaces of higher dimension, we grow the space tree from the current selected restricted space $\cS_{k^\ast}$ every $N'$ episodes. 
{
If $\cS_{k^\ast}$ is $l$-dimensional, we add restricted spaces that are $(l+1)$-dimensional and contain $\cS_{k^*}$ as child nodes of $\cS_{k^\ast}$. 
Formally, we initialize the space tree $T_{\text{space}}^0$ via
\begin{equation}
    T_{\text{space}}^0 = \{{\cal S}_k : |k|=1, k \in P(\{1, \dots, M\})\},
    \label{eq:init_tree}
\end{equation}
where $M$ denotes the dimension of the full state space and $P$ denotes the power set. 
Let $T_\text{space}^{(h)}$ and ${\cal S}_{k^\ast}$ denote the space tree after the $h$-th expansion and the current selected restricted space respectively. 
Note, ${\cal S}_{k^\ast}$ is sampled according to~\equref{eq:sample} and no domain knowledge is used for selecting ${\cal S}_{k^\ast}$.
The space tree after the $(h+1)$-th expansion is 
\begin{equation}
\begin{aligned}
    T_\text{space}^{(h+1)} &= T_\text{space}^h \cup \\ 
    & \{{\cal S}_k : |k|=|k^\ast |+1,k^\ast\subset k, k \in P(\{1,\dots, M\})\},
\end{aligned}
\end{equation}
\ie, all restricted spaces that are ($|k^\ast |+1$)-dimensional and contain $S_{k^\ast}$ are added. The counters associated with the new restricted spaces are initialized from states in the replay buffer. %
}
Specifically, for each newly added restricted space ${\cal S}_{k^*}$, we initialized the corresponding counter to be
\begin{equation}
c_{k^*}(s_{k^*}) = \sum_{s \in \cD} \mathbbm{1}[\text{proj}_{k^*}(s) = s_{k^*}],
\label{eq:count}
\end{equation}
where $\mathbbm{1}[\cdot]$ is the indicator function (1 if argument is true; 0 otherwise). {Once a restricted space was added, we successively increment the counter,} {
\ie, the counter $c_{k^*}$ isn't recomputed from scratch every episode.
} 
This ensures that updating of counters is efficient. Goal selection, space selection and tree expansion are summarized in \algref{alg:select_goal}.

\subsection{Counting in Continuous State Spaces}
\label{subsec:continuous_count}
For high-dimensional continuous state spaces, the counters in CMAE could be implemented using counting methods such as neural density models~\cite{Ostrovski17, Bellemare17} or hash-based counting~\cite{Tang17}. Both  approaches have been shown to be effective for  counting in continuous state spaces. 

In our implementation, we adopt hash-based counting~\cite{Tang17}. Hash-based counting discretizes the state space via a hash function $\phi(s): S \rightarrow \mathbb{Z}$, which maps a given state $s$ to an integer that is used to index into a table.
{Specifically, in CMAE, each restricted space $\cS_k$ is associated with a hash function $\phi_k(s): \cS_k \rightarrow \mathbb{Z}$, which maps the continuous $s_k$ to an integer. The hash function $\phi_k$ is used when CMAE updates or queries the counter associated with restricted space $\cS_k$. }
Empirically, we found CMAE with hash-counting  
to perform  well in environments with continuous state spaces.

\subsection{Analysis}
To provide more insights into how the proposed method improves  data efficiency, we analyze the two major components of CMAE: (1) shared goal exploration and (2) restricted space exploration on a simple multi-player matrix game. 
We first define the multi-player matrix game: 
\begin{example} In a cooperative $2$-player $l$-action matrix game~\cite{Myerson2013}, a payoff matrix $U \in \mathbb{R}^{l \times l}$, which is unobservable to the players, describes the payoff the agents obtain after an action configuration is executed. The agents' goal is to find the action configuration that maximizes the collective payoff. 
\label{ex:matrix}
\end{example}
To efficiently find the action configuration that results in maximal payoff, the agents need to uniformly try different action configurations. We show that exploration with shared goals enables agents to see all distinct action configurations more efficiently than exploration without a shared goal. 
{Specifically, when  exploring  without  shared  goal,  the  agents don’t  coordinate  their  behavior.   It  is  equivalent  to  uniformly picking one action configuration from all configurations.
When  performing  exploration  with  a shared goal,  the  least visited  action  configuration  will  be  chosen  as  the  shared goal.  The two agents coordinate to choose the actions that achieve the goal at each step, making exploration more efficient.}
The following claim formalizes this: 

{\begin{restatable}{claim}{goal}{Consider the $2$-player $l$-action matrix game in~\exref{ex:matrix}. Let $m = l^2$ denote the total number of action configurations.  Let $T_m^{\text{share}}$ and $T_m^{\text{non-share}}$ denote the number of steps needed to see all $m$ action configurations at least once for exploration with shared goal and for exploration without shared goal respectively. Then we have $\expectation[T_m^{\text{share}}] = m$ and $\expectation[T_m^{\text{non-share}}] = m\sum_{i=1}^m \frac{1}{i} = \Theta(m \ln m)$.\footnote{
$\Theta(g)$ means asymptotically bounded above and below by $g$.}
}
\label{clm:goal}\end{restatable}}
\begin{proof}
See supplementary material. 
\end{proof}

Next,  we show that whenever the payoff matrix 
depends only on one agent's action 
the expected number of steps  to see the maximal reward can be further reduced by first exploring restricted spaces. 

{\begin{restatable}{claim}{subspace}{Consider a special case of~\exref{ex:matrix} where the payoff matrix depends only on one agent's action. Let $T^{\text{sub}}$  denote the number of steps needed to discover the maximal reward when exploring the action space of agent one and agent two independently. Let $T^{\text{full}}$ denote the number of steps needed to discover the maximal reward when the full action space is explored. 
Then, we have $T^{\text{sub}} = \gO(l)$ and $T^{\text{full}} = \gO(l^2)$.
}
\label{clm:subspace}\end{restatable}}
\begin{proof}
See supplementary material. 
\end{proof}

Suppose the payoff matrix depends on all agents' actions. In this case, CMAE will move to the full space after the restricted spaces are well-explored. For this, the expected total number of steps to see the maximal reward is $\gO(l + l^2) = \gO(l^2)$.

%% file: exp.tex

\begin{table*}[t]
\centering
\begin{tabular}{l|ccccc}
\specialrule{.15em}{.05em}{.05em}
          &  CMAE (Ours)   & Q-learning & \makecell{Q-learning + Bonus} &  EITI & EDTI \\
\hline
\hline
Pass-sparse &        \textbf{1.00$\pm$0.00}    &    0.00$\pm$0.00 & 0.00$\pm$0.00 & 0.00$\pm$0.00 & 0.00$\pm$0.00 \\
Secret-Room-sparse &   \textbf{1.00$\pm$0.00}    &    0.00$\pm$0.00 & 0.00$\pm$0.00 & 0.00$\pm$0.00 & 0.00$\pm$0.00 \\
Push-Box-sparse  &      \textbf{1.00$\pm$0.00}    &    0.00$\pm$0.00 & 0.00$\pm$0.00 & 0.00$\pm$0.00 & 0.00$\pm$0.00 \\
\hline
Pass-dense &        \textbf{5.00$\pm$0.00}    &    1.25$\pm$0.02  &      1.42$\pm$0.14 & 0.00$\pm$0.00 &0.18$\pm$0.01  \\
Secret-Room-dense &   \textbf{4.00$\pm$0.57}  &   1.62$\pm$0.16   &     1.53$\pm$0.04 & 0.00$\pm$0.00 &0.00$\pm$0.00  \\
Push-Box-dense  &      1.38$\pm$0.21  &    \textbf{1.58$\pm$0.14}  & 1.55$\pm$0.04  & 0.10$\pm$0.01 & 0.05$\pm$0.03 \\

\specialrule{.15em}{.05em}{.05em}
\end{tabular}
\caption{\emph{Final metric} of episode rewards of CMAE and baselines on sparse-reward {(top)} and dense-reward {(bottom)} MPE tasks. }
\label{tb:mpe}
\end{table*}

\begin{figure*}[t]

\centering
\renewcommand{\arraystretch}{0}
\begin{tabular}{ccc}
\includegraphics[width=0.33\textwidth]{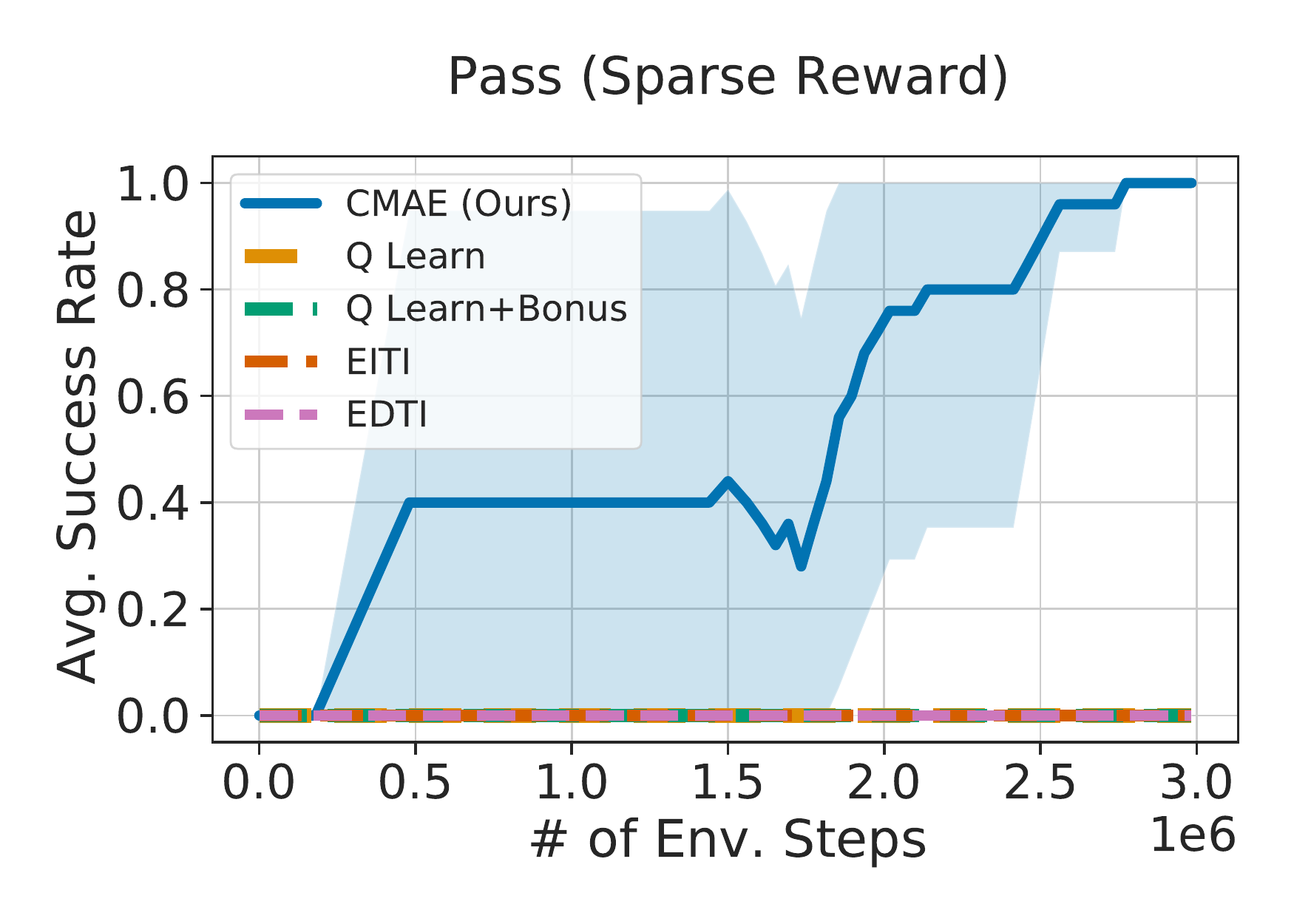}
&
\hspace{-0.6cm}
\includegraphics[width=0.33\textwidth]{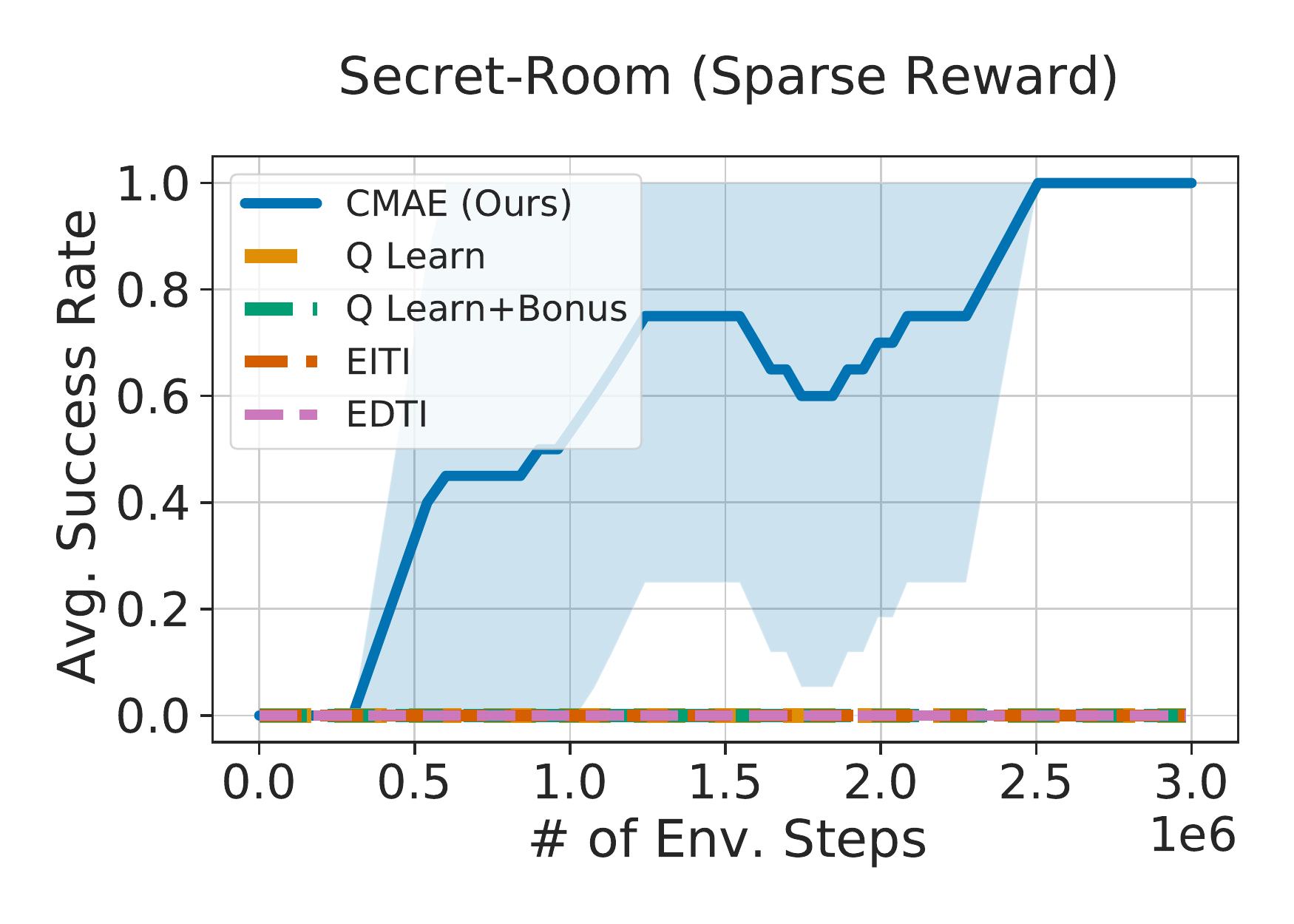}
&
\hspace{-0.6cm}
\includegraphics[width=0.33\textwidth]{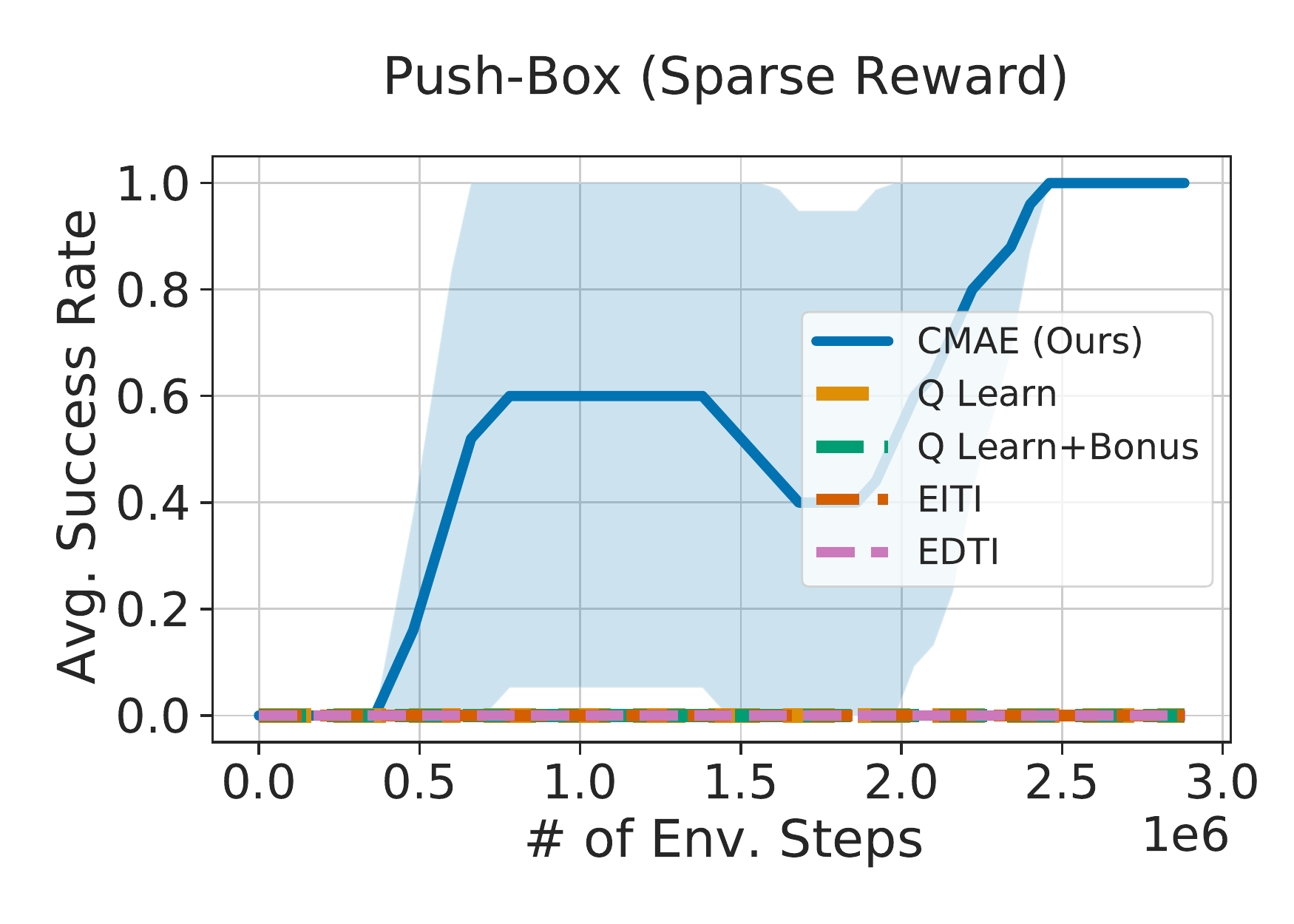}
\\ 
\includegraphics[width=0.33\textwidth]{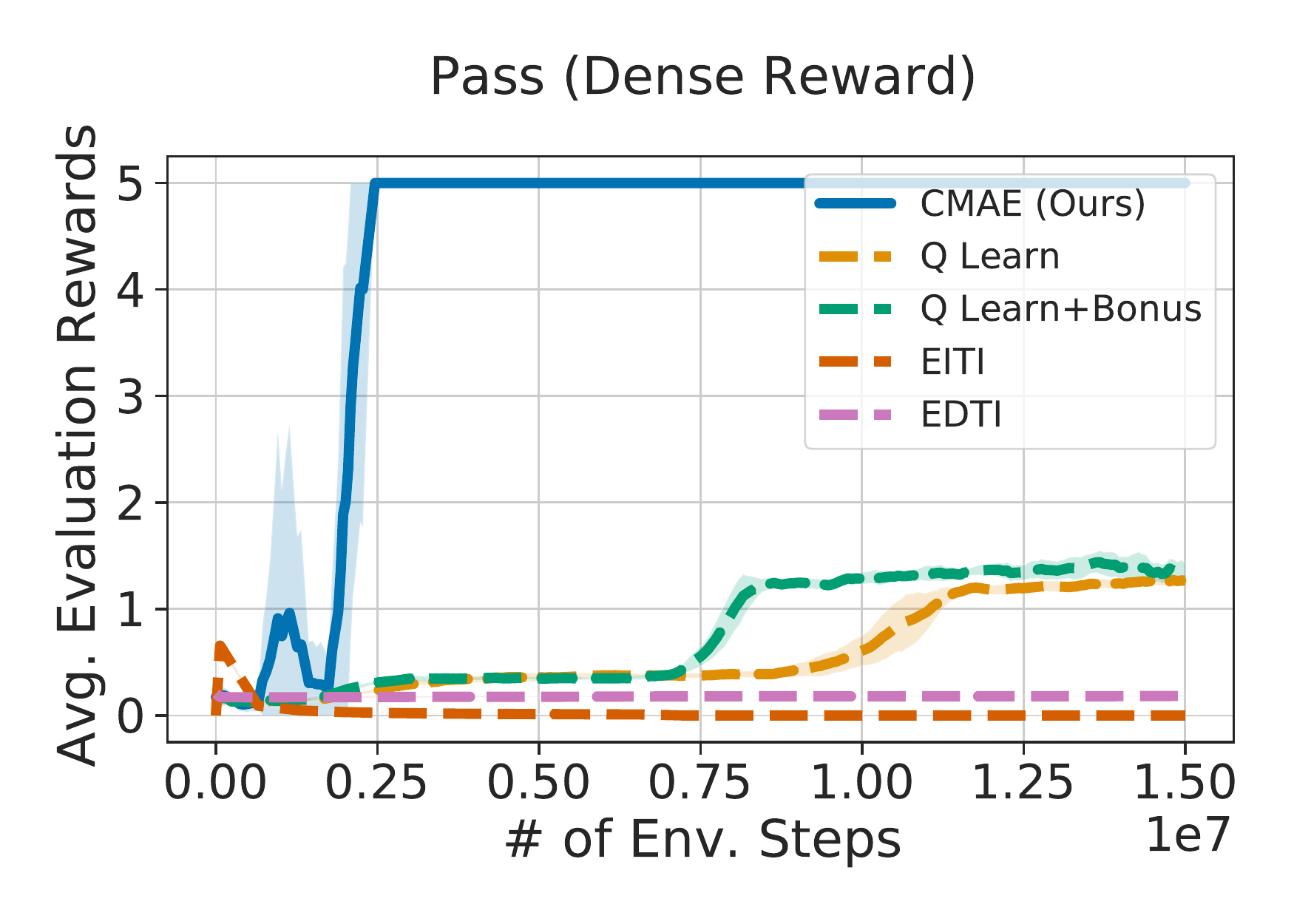}
&
\hspace{-0.6cm}
\includegraphics[width=0.33\textwidth]{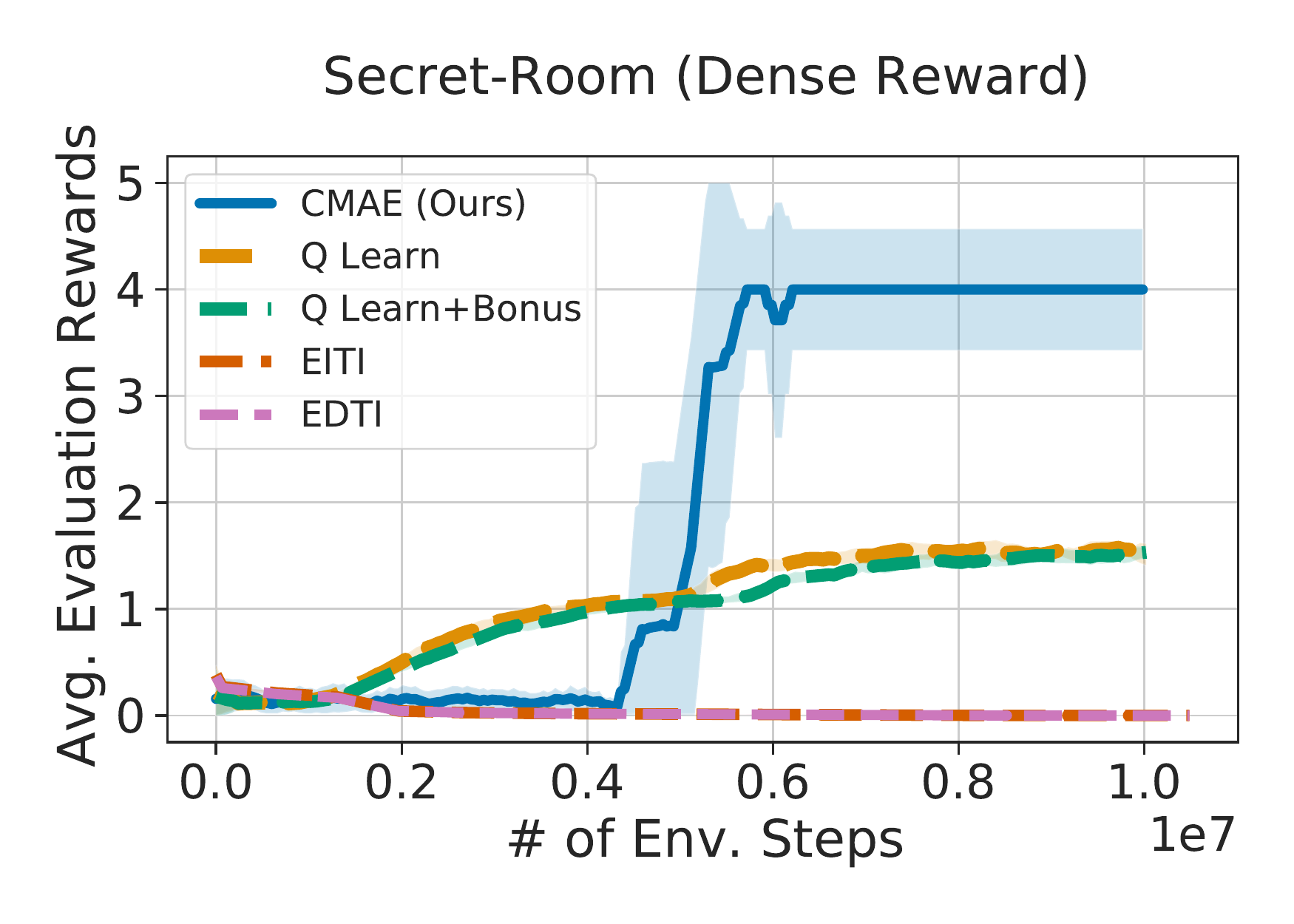}
&
\hspace{-0.6cm}
\includegraphics[width=0.33\textwidth]{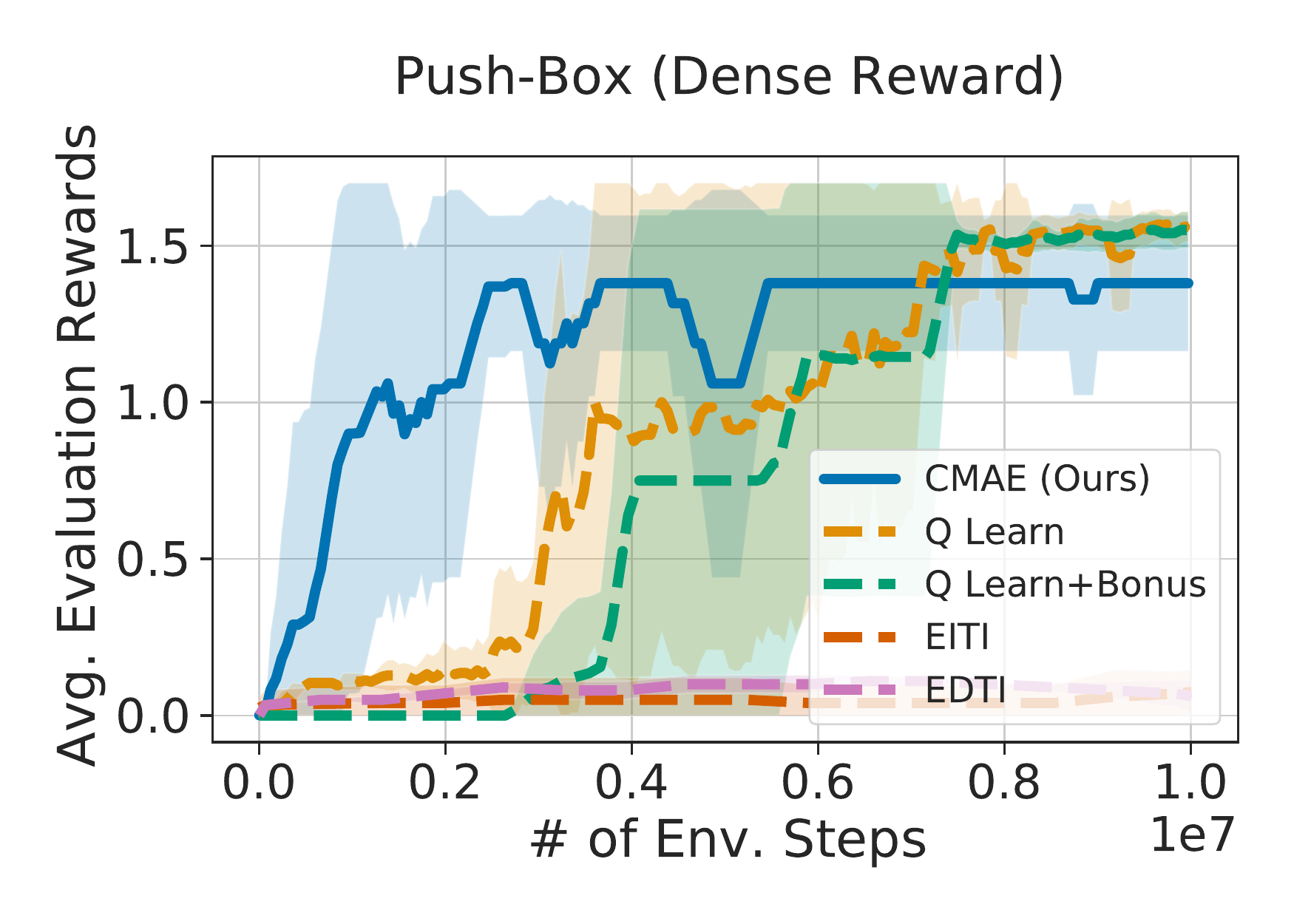}

\end{tabular}%
\caption{Training curves on sparse-reward and dense-reward MPE tasks.} 
\label{fig:mpe_plot}
\end{figure*}

\begin{table*}[t]
\centering
\begin{tabular}{l|ccccc}
\specialrule{.15em}{.05em}{.05em}
          &  CMAE (Ours)  &  Weighted QMIX &  Weighted QMIX + Bonus & QMIX & QMIX + Bonus \\
\hline
\hline
3m-sparse &        \textbf{47.7$\pm$35.1}     & 2.7$\pm$5.1 &11.5$\pm$8.6 &    0.0$\pm$0.0  &      11.7$\pm$16.9\\
2m\_vs\_1z-sparse &   \textbf{44.3$\pm$20.8}   &0.0$\pm$0.0 &19.4$\pm$18.1 &   0.0$\pm$0.0   &     19.8$\pm$14.1 \\
3s\_vs\_5z-sparse  &      0.0$\pm$0.0  &    0.0$\pm$0.0  & 0.0$\pm$0.0  & 0.0$\pm$0.0 &0.0$\pm$0.0 \\
\hline
3m-dense &        98.7$\pm$1.7    &    98.3$\pm$2.5  &      \textbf{98.9$\pm$1.7} & 97.9$\pm$3.6 &97.3$\pm$3.0  \\
2m\_vs\_1z-dense &   98.2$\pm$0.1  &   \textbf{98.5$\pm$0.1}   &     96.0$\pm$1.8 & 97.1$\pm$2.4 &95.8$\pm$1.7  \\
3s\_vs\_5z-dense  &      81.3$\pm$16.1  &    92.2$\pm$6.6  & \textbf{95.3$\pm$2.2}  & 75.0$\pm$17.6 &78.1$\pm$24.4 \\
\specialrule{.15em}{.05em}{.05em}
\end{tabular}
\caption{\emph{Final metric} of success rate ($\%$) of CMAE and baselines on sparse-reward (top) and dense-reward (bottom) SMAC tasks. }
\label{tb:smac}
\end{table*}

\vspace{-0.0cm}
\section{Experimental Results}
\label{sec:experimental}
\vspace{-0.1cm}
We evaluate  CMAE on two challenging environments: (1) {a discrete version of} the multiple-particle environment (MPE)~\cite{maddpg, eiti}; and (2) the  Starcraft multi-agent challenge (SMAC)~\citep{smac}. In both environments, we consider both dense-reward and sparse-reward settings. 
In a sparse-reward setting, agents don't receive any intermediate rewards, \ie, agents only receive a reward when a task is completed.

\textbf{Tasks:}
We first consider the following tasks of the sparse-reward MPE environment:

\emph{Pass-sparse}: Two agents operate within two rooms of a $30 \times 30$ grid. There is one switch in each room. The rooms are separated by a door and agents start in the same room. The door will open only when one of the switches is occupied. 
The agents see collective positive reward and the episode terminates only when both agents changed to the other room. {The state vector contains  $x, y$ locations of all agents and binary variables to indicate if  doors are open. 
}

\emph{Secret-Room-sparse}: Secret-Room  extends  \emph{Pass}. There are two agents and four rooms. One large room on the left and three small rooms on the right. There is one door between each small room and the large room. The switch in the large room controls all three doors. The switch in each small room only controls the room's door. The agents need to navigate to one of the three small rooms, \ie, the target room, to receive positive reward. The grid size is $25 \times 25$.  The task is considered solved if both agents are in the target room. {The state vector contains  $x, y$ locations of all agents and binary variables to indicate if  doors are open.
}

\emph{Push-Box-sparse}: There are two agents and one box in a $15\times15$ grid. Agents need to push the box to the wall to receive positive reward. The box is heavy, so both agents need to push the box in the same direction at the same time to move the box.  The task is considered solved if the box is pushed to the wall. {The state vector contains  $x, y$ locations of all agents and the box.}

{For further details on the sparse-reward MPE tasks, please see~\citet{eiti}.} For completeness, in addition to the aforementioned sparse-reward setting, we also consider a dense-reward version of the three tasks. Please see Appendix~\ref{sec:sup_mpe}  for more details on the environment settings.

To evaluate CMAE on environments with continuous state space, 
we consider three standard tasks in SMAC~\citep{smac}:  \emph{3m}, \emph{2m\_vs\_1z}, and \emph{3s\_vs\_5z}. While the tasks are considered challenging, the commonly used reward is dense, \ie, carefully \textit{hand-crafted} intermediate rewards are used to guide the agents' learning. However, in many real-world applications, designing effective intermediate rewards may be very difficult or infeasible.  Therefore, in addition to the dense-reward setting, we also {consider the sparse-reward setting specified by the SMAC environment~\cite{smac} for the three tasks.} {In SMAC, the  state  vector  contains for all units on the map: $x, y$ locations, health, shield, and unit type. Note SMAC tasks are partially observable, \ie,  agents only observe  information of units within a range.}
Please see Appendix~\ref{sec:sup_smac}  for more details on the SMAC environment.

\textbf{Experimental Setup:} 
For MPE tasks, we combine CMAE with  Q-learning~\cite{SuttonRL, dqn1, dqn2}. We compare CMAE with exploration via information-theoretic influence (EITI) and exploration via decision-theoretic influence (EDTI)~\citep{eiti}. 
EITI and EDTI
results are obtained using the {publicly available code released by the authors.} 

For a more complete comparison, we also show the results of Q-learning with $\epsilon$-greedy and Q-learning with count-based exploration~\cite{Tang17}, where exploration bonus is given when a novel state is visited.

For SMAC tasks, we combine CMAE with QMIX~\citep{qmix}. We compare with QMIX~\citep{qmix}, QMIX with count-based exploration, weighted QMIX~\citep{wqmix}, and weighted QMIX with count-based exploration~\cite{Tang17}. {For QMIX and weighted QMIX, we use the publicly available code released by the authors.} {In all  experiments we use restricted spaces of less than four dimensions.}

Note, to increase efficiency of the baselines with count-based exploration, in both MPE and SMAC experiments, the counts are shared across all agents. 
{We use `+Bonus' to refer to a baseline with count-based exploration.}

\textbf{Evaluation Protocol:}
To ensure a rigorous and fair evaluation, we follow the evaluation protocol suggested by \citet{Henderson17, Colas18}. We evaluate the target policies in an independent evaluation environment and report \emph{final metric}. The \emph{final metric} is  an average episode reward or success rate over the last 100 evaluation episodes, \ie, 10 episodes for each of the last ten policies during training.  We repeat all experiments using five runs with different random seeds. 

Note that EITI and EDTI~\citep{eiti} report the episode rewards \vs the number of model updates   as an evaluation metric. This isn't   common  when evaluating RL algorithms as this plot doesn't reflect an RL approach's data efficiency. In contrast, the episode reward \vs number of environment steps is a more common metric 
for
data efficiency and is adopted by many RL works~\cite{ddpg, acktr, dqn1, dqn2, her, ModelShen20, kf, iswitch}, particularly works on RL exploration~\cite{maven, aiga20, Pathak17, Tang17, OptimisticRashid20}. Therefore, following most prior works, we report the episode rewards \vs the number of environment steps.

\begin{figure}[t]
\centering

\includegraphics[width=0.48\textwidth]{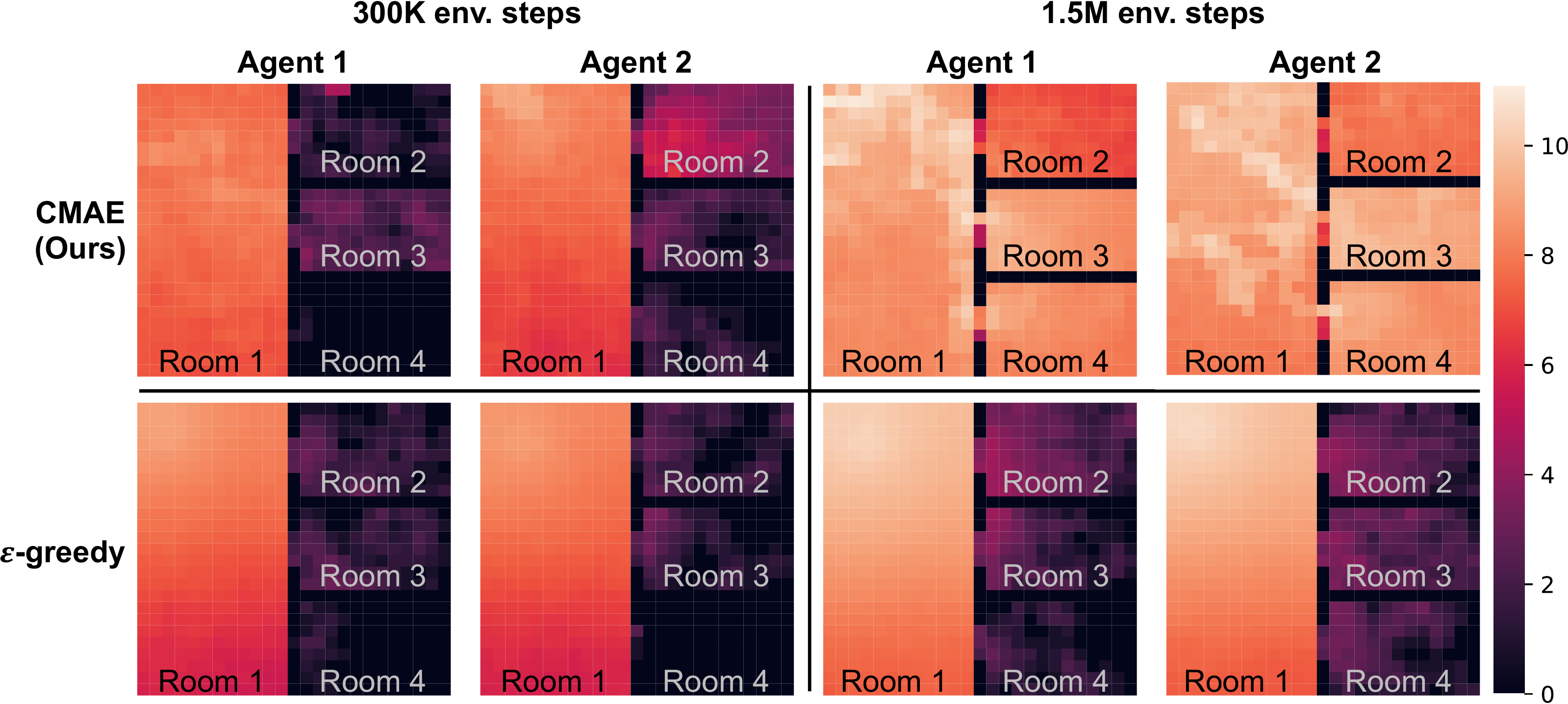}
\caption{Visitation map {(log scale)} of CMAE (top) and $\epsilon$-greedy (bottom) on {the} \emph{Secret-Room} {task}. } 
\label{fig:visit_map}
\vspace{0.3cm}
\end{figure}

\begin{figure*}[t]

\centering
\begin{tabular}{cccc}
\includegraphics[width=0.33\textwidth, height=4cm]{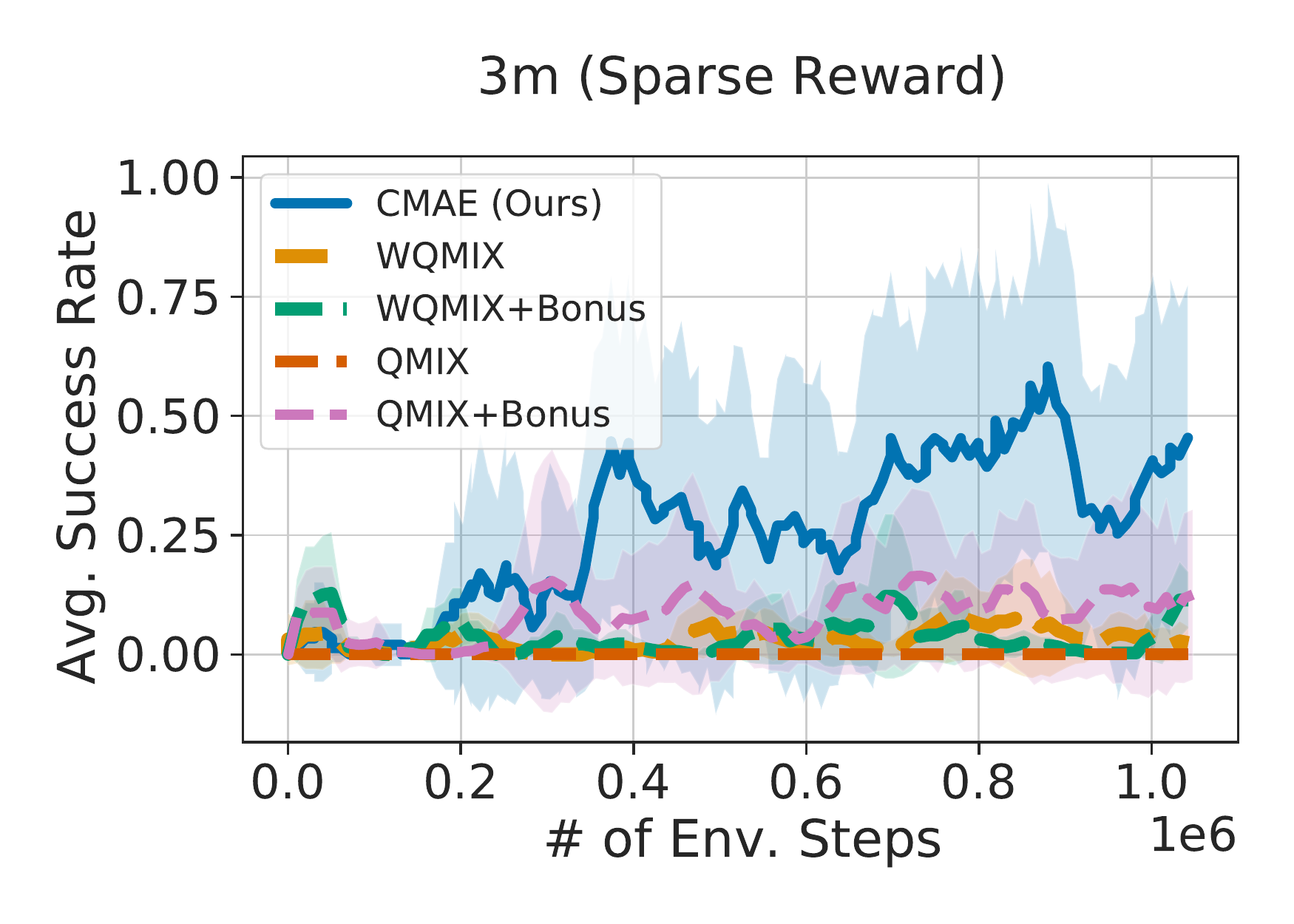}
&
\hspace{-0.6cm}
\includegraphics[width=0.33\textwidth, height=4cm]{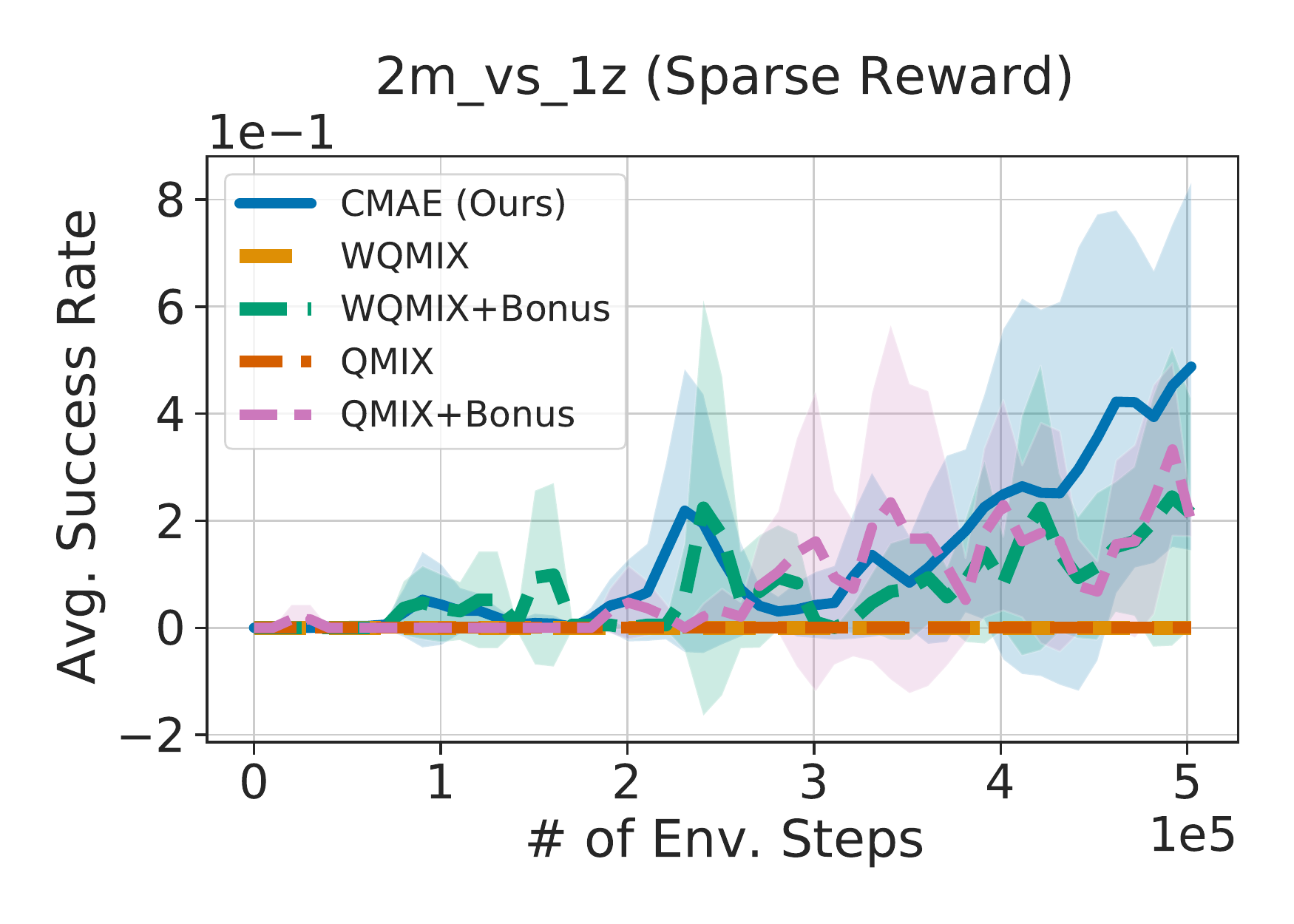}
&
\hspace{-0.6cm}
\includegraphics[width=0.33\textwidth, height=4cm]{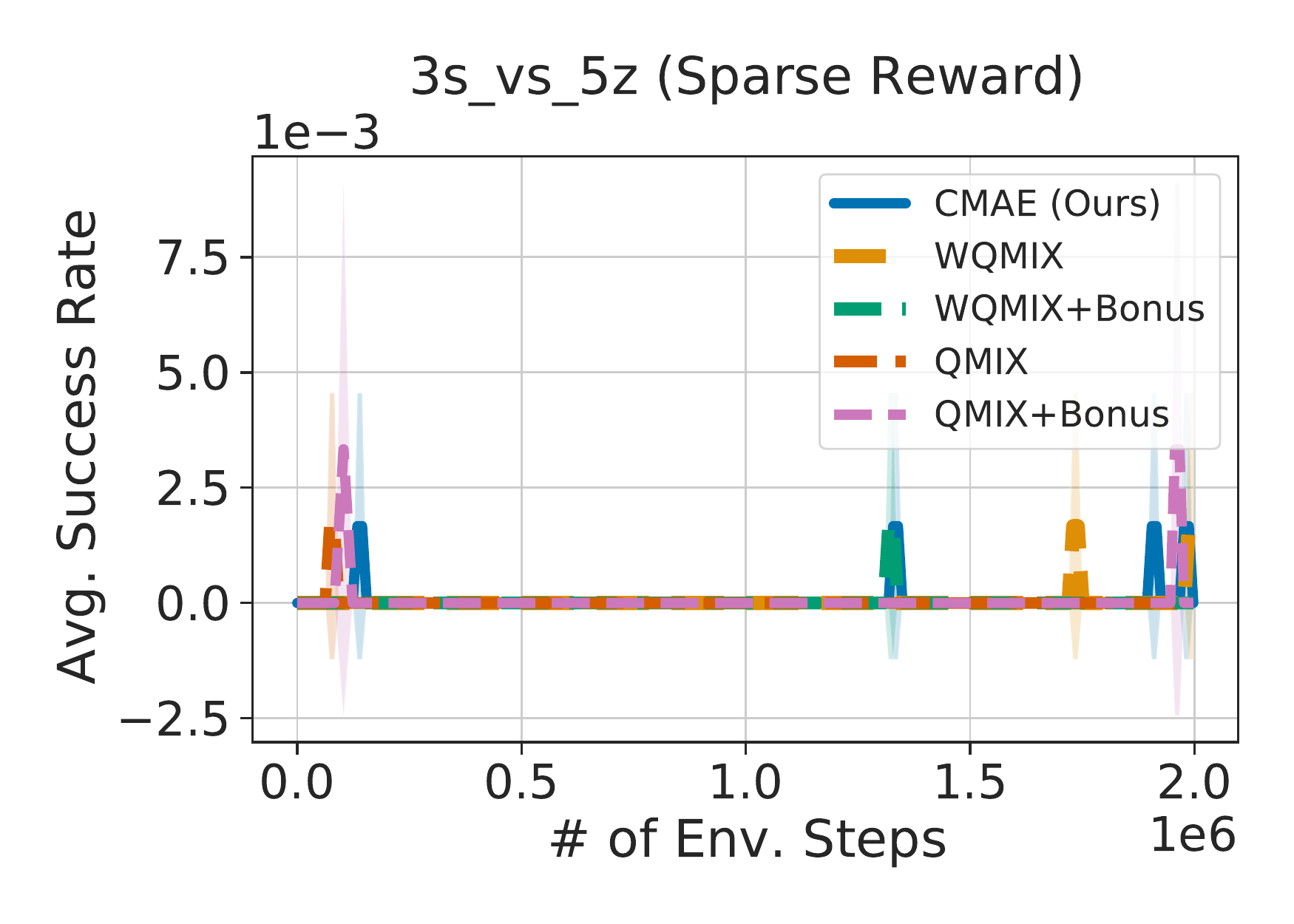}

\end{tabular}%
\caption{Training curves on sparse-reward SMAC tasks.} 
\label{fig:smac_plot}
\end{figure*}

\textbf{Results:} We first compare CMAE with baselines on \emph{Pass}, \emph{Secret Room}, and \emph{Push-Box}. 
The \emph{final metric} and {standard deviation} is reported in~\tabref{tb:mpe}.In
the sparse-reward setting, only CMAE is able to solve the tasks, while all baselines do not learn a meaningful policy. Corresponding training curves {with standard deviation} are included in~\figref{fig:mpe_plot}. CMAE achieves a $100\%$ success rate on \emph{Pass}, \emph{Secret-room}, and \emph{Push-Box} within 3M environment steps. In contrast, baselines cannot solve the task within the given step budget of 3M steps.

Recently,~\citet{aiga20}
{pointed} out
that many existing exploration strategies excel in challenging sparse-reward tasks but fail in simple tasks that can be solved by using classical exploration methods such as $\epsilon$-greedy. To ensure CMAE doesn't fail in simpler tasks, we run experiments on the  dense-reward version of the three tasks. As shown in~\tabref{tb:mpe}, CMAE achieves similar or better performance than the baselines in dense-reward settings.

We also compare the exploration behavior of CMAE to Q-learning with $\epsilon$-greedy exploration using the \emph{Secret-Room} environment. 
{The visit count (in log scale) of each location is visualized in~\figref{fig:visit_map}.}
{In early stages of training,}
both CMAE (top) and $\epsilon$-greedy (bottom) explore only locations in the left room. However, after $1.5$M steps,
{CMAE agents frequently visit the three rooms on the right while $\epsilon$-greedy agents mostly remain within the left room.}

On SMAC, we first compare CMAE with baselines in the sparse-reward setting. 
{
Since the number of nodes in the space tree grows combinatorially, discovering useful high-dimensional restricted spaces for tasks with high-dimensional state space, such as SMAC, may be infeasible.
However, we found empirically that exploring of low-dimensional restricted spaces is already beneficial in a subset of SMAC tasks. }
The results on SMAC tasks are summarized in~\tabref{tb:smac}, where final metric and standard deviation of evaluation success rate is reported. As shown in~\tabref{tb:smac} (top), in \emph{3m-sparse} and \emph{2m\_vs\_1z-sparse}, QMIX and weighted QMIX, which rely on $\epsilon$-greedy exploration, rarely solve the task.
{When combined with count-based exploration, both QMIX and weighted QMIX}
are able to achieve $18\% \text{ to } 20\%$ success rate. CMAE achieves much higher success rate of $47.7\%$ and $44.3\%$ on \emph{3m-sparse} and \emph{2m\_vs\_1z-sparse}, respectively. Corresponding training curves {with standard deviation} are included in~\figref{fig:smac_plot}.
We also run experiments on dense-reward SMAC tasks, where handcrafted intermediate rewards are available. As shown in~\tabref{tb:smac} (bottom), {CMAE} achieves similar performance to state-of-the-art baselines in dense-reward SMAC tasks.

{\textbf{Limitations:}} To show limitations of the proposed method, we run experiments on the sparse-reward version of \emph{3s\_vs\_5z}, which is classified as `hard' even in the dense-reward setting~\cite{smac}. As shown in~\tabref{tb:smac} and \figref{fig:smac_plot}, CMAE as well as all baselines fail to solve the task. In \emph{3s\_vs\_5z}, the only winning strategy is to {force the enemies to scatter around the map} and attend to them one by one~\cite{smac}. Without hand-crafted intermediate reward, we found it to be extremely challenging for any approach to pick up this strategy. This demonstrates that efficient exploration for MARL in sparse-reward settings is still a very challenging and open problem, which requires more attention from the community.

\textbf{Ablation Study:}
\begin{figure}[t]
\centering
\includegraphics[width=0.38\textwidth]{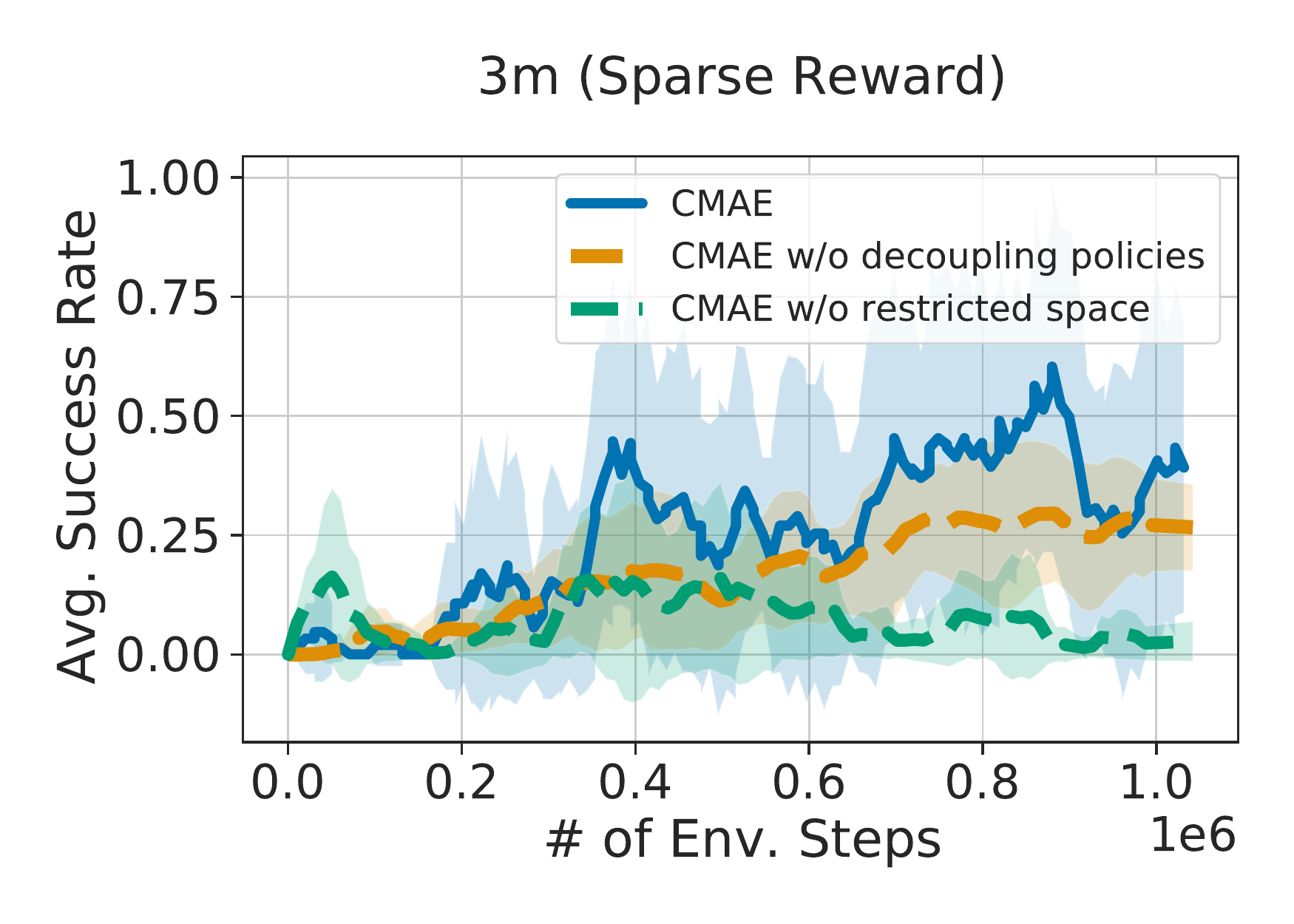}
\caption{Ablation: CMAE, CMAE without decoupling target and exploration policies, and CMAE without restricted space exploration on \emph{3m-sparse}.} 
\label{fig:ablation}
\vspace{0.2cm}
\end{figure}
To better understand the approach, we perform an ablation study to examine the effectiveness of the proposed (1) target and exploration policy decoupling and (2) restricted space exploration. We conduct the experiments on \emph{3m-sparse}. As~\figref{fig:ablation} shows, without decoupling the exploration and target policies, the success rate drops from $47.7\%$ to $25.4\%$. In addition, without restricted space exploration, \ie, by directly exploring the full state space, the success rate drops to $9.4\%$. This demonstrates that the restricted space exploration and policy decoupling are essential to CMAE's success.

%% file: rel.tex
\section{Related Work}
\label{sec:related}

We discuss recently developed methods for exploration in reinforcement learning, multi-agent reinforcement learning, and concurrent reinforcement learning subsequently.

\noindent{\bf Exploration for Deep Reinforcement Learning:} 
A wide variety of exploration techniques for deep reinforcement learning have been studied, deviating from classical noise-based methods.
Generalization of count-based approaches, which give near-optimal results in tabular reinforcement learning, to environments with continuous state spaces have been proposed. For instance, \citet{Bellemare17} propose  a density model to measure the agent's uncertainty.  
Pseudo-counts are derived from the density model which give rise to an exploration bonus encouraging assessment of rarely visited states.  Inspired by \citet{Bellemare17},~\citet{Ostrovski17} discussed  a neural density model,  
to estimate the pseudo count, and \citet{Tang17} use a hash function to estimate the count.  

 Besides count-based approaches, meta-policy gradient~\cite{Xu18} uses the target policy's improvement as the reward to train an exploration policy. The resulting exploration policy differs from the actor policy, and  enables more global exploration. \citet{Stadie06} propose an exploration strategy based on assigning an exploration bonus from a concurrently learned environment model.  {\citet{Lee20} cast exploration  as a state marginal matching (SMM) problem and aim to learn a policy for which the state marginal distribution matches a uniform distribution.}
 Other related works on exploration include curiosity-driven exploration~\cite{Pathak17}, diversity-driven exploration~\cite{Hong18}, GEP-PG~\cite{Colas18}, $EX^2$~\cite{Fu17}, bootstrap DQN~\cite{Osband16} and random network distillation~\cite{Burda19}. 
In contrast to our approach, all the techniques mentioned above target  single-agent deep reinforcement learning.

\noindent{\bf Multi-agent Deep Reinforcement Learning (MARL):}  
MARL~\cite{maddpg, Foerster17, pic, wqmix, Jain20, Zhou20, Christianos20, hts-rl, Jain21, updet} has drawn much attention recently. 
MADDPG~\cite{maddpg} uses a central critic that considers other agents' action policies to handle the non-stationary environment issues in the multi-agent setting. DIAL~\cite{Foerster16} uses an end-to-end differentiable architecture that allows agents to learn to communicate. \citet{Jiang18} propose an attentional communication model that learns when communication is helpful for a cooperative setting. \citet{Foerster17} add a `fingerprint' to each transition tuple in the replay memory to track the age of the transition tuple and  stabilize training. In `Self-Other-Modeling' (SOM)~\cite{Raileanu18},  an agent uses its own policy to predict other agents' behavior and  states. 

 While inter-agent communication~\cite{Priya20, Rangwala20, SuccinctZhang20, Ding20, Jiang18, Foerster16, qmix, Omidshafiei17, Jain19} has been considered,  for exploration,  multi-agent 
 approaches  rely on classical noise-based exploration.  As  discussed in~\secref{sec:intro}, a noise-based approach prevents the agents from sharing their understanding of the environment. A team of cooperative agents with a noise-based exploration policy can only explore local regions that are close to their individual actor policy, which contrasts the approach from CMAE. 

Recently, approaches that consider coordinated exploration have been proposed. Multi-agent variational exploration (MAVEN)~\citep{maven} introduces a latent space for hierarchical control. Agents condition their behavior on the latent variable to perform committed exploration.  
Influence-based exploration~\citep{eiti} captures the influence of one agent's behavior on others. Agents are encouraged to visit `interaction points' that will change other agents' behaviour. 

\noindent{\bf Concurrent Deep Reinforcement Learning:}  
\citet{concurrent1} study coordinated exploration in concurrent reinforcement learning, maintaining an environment model and extending posterior sampling such that agents  explore in a coordinated fashion. {\citet{cmrl} proposed concurrent meta reinforcement learning (CMRL) which permits a set of parallel agents to communicate with each other and find efficient exploration strategies.}
The concurrent setting   differs from the multi-agent setting of our approach. 
In a concurrent setting, agents operate in different instances of an environment, \ie, one agent's action has no effect on the observations and rewards received by other agents. In contrast, in the multi-agent setting, agents share the same instance of an environment. An agent's action  changes observations and rewards observed by other agents.

%% file: conc.tex
\section{Conclusion}
\label{sec:conclusion}
We propose cooperative multi-agent exploration (CMAE). It defines shared goals and learns coordinated exploration policies. To find a goal for efficient exploration we study restricted space selection which helps, particularly in sparse-reward environments. Empirically, we demonstrate that CMAE increases exploration efficiency. We hope this is  a first step toward efficient coordinated  exploration. 

\textbf{Acknowledgement:}
This work is supported in part by NSF under Grant $\#1718221$, $2008387$, $2045586$, and MRI $\#1725729$, UIUC,
Samsung, Amazon, 3M, and Cisco Systems Inc. RY is supported by a Google Fellowship. 

%% file: supp.tex

{\centering \Large \textbf{Appendix: Cooperative Exploration for Multi-Agent Deep Reinforcement Learning}}

In this appendix we first provide the proofs for~\clmref{clm:goal} and~\clmref{clm:subspace} in~\secref{sec:sup_clm1} and~\secref{sec:sup_clm2}. We then provide  information regarding the MPE and SMAC environments (\secref{sec:sup_mpe},~\secref{sec:sup_smac}), implementation details (\secref{sec:sup_implement}), and the absolute metric (\secref{sec:sup_abs_metric}). Next, we provide additional results on MPE tasks (\secref{sec:sup_island}), additional results of baselines (\secref{sec:sup_more_steps}) and training curves (\secref{sec:sup_curve}).

\section{Proof of~\clmref{clm:goal}}
\label{sec:sup_clm1}
\goal*
\begin{proof}
When exploring without shared goal, the agents don't coordinate their behavior. It is equivalent to uniformly picking one action configuration from the $m$ configurations. We aim to show after $T_{m}^{\text{non-share}}$ time steps, the agents tried all $m$ distinct action configurations. 
Let $T_i$ be the number of steps to observe the $i$-th distinct action configuration after seeing $i-1$ distinct configurations. Then 
\be
\expectation[T_{m}^{\text{non-share}}] = \expectation[T_1] + \dots +	\expectation[T_m].
\ee
In addition, let $P(i)$ denotes the probability of observing the $i$-th distinct action configuration after observing $i - 1$ distinct configurations. We have 
\be
P(i) = 1 - \frac{i - 1}{m} = \frac{m - i + 1}{m}.
\ee
Note that $T_i$ follows a geometric distribution with success probability $P(i) = \frac{m - i + 1}{m}$. 
Then the expected number of timesteps to see the $i$-th distinct configuration after seeing $i-1$ distinct configurations is 
\be
\expectation[T_i] =  \frac{m}{m - i + 1}.
\ee
Hence, we obtain 
\begin{equation}
\begin{aligned}
\expectation[T_{m}^{\text{non-share}}] &= \expectation[T_1] + \dots +	\expectation[T_m] \\
        &= \sum_{i=1}^m  \frac{m}{m - i + 1} \\
        &= m \sum_{i=1}^m  \frac{1}{i}.
\label{eq:mn}
\end{aligned}
\end{equation}
From calculus, $\int_1^m \frac{1}{x} dx = \ln m$. Hence we obtain the following inequality
\begin{equation}
\sum_{i=1}^m \frac{1}{i + 1} \le \int_1^m \frac{1}{x} dx = \ln m \le \sum_{i=1}^m \frac{1}{i}.
\label{eq:theta}
\end{equation}
From~\equref{eq:theta}, we obtain $ \sum_{i=1}^m \frac{1}{i} = \mathcal{O} (\ln m)$\footnote{
$\mathcal{O}(g)$ means asymptotically bounded above by $g$.} and $ \sum_{i=1}^m \frac{1}{i} = \Omega(\ln m)$\footnote{
$\Omega(g)$ means asymptotically bounded below by $g$.}, which implies 
\begin{equation}
\sum_{i=1}^m \frac{1}{i} = \Theta(\ln m).
\label{eq:ln}
\end{equation}
Combining~\equref{eq:mn} and~\equref{eq:ln}, we get $\expectation[T_{m}^{\text{non-share}}]  = \Theta(m \ln m)$.

When performing exploration with shared-goal, the least visited action configuration will be chosen as the shared goal. The two agents coordinate to choose the actions that achieve the goal at each step. Hence, at each time step, the agents are able to visit a new action configuration. Therefore, exploration with shared goal needs $m$ timesteps to visit all $m$ action configurations, \ie, $T_m^{\text{share}} = m$, which completes the proof. 
\end{proof}

\section{Proof of~\clmref{clm:subspace}}
\label{sec:sup_clm2}
\subspace*

\begin{proof}
When we explore the action spaces of agent one and agent two independently, there are $2l$ distinct action configurations ($l$ action configurations for each agent) to explore. Since the reward function depends only on one agent's action, one of these $2l$ action configurations must lead to the maximal reward. Therefore, by checking distinct action configurations at each time step, we need at most $2l$ steps to receive the maximal reward, \ie,  $\expectation[T^{\text{sub}}] = \gO(l)$. 

In contrast, when we explore the joint action space of agent one and agent two. There are $l^2$ distinct action configurations. Because the reward function depends only on one agent's action, $l$ of these $l^2$ action configurations must lead to the maximal reward. In the worst case, we choose the $l^2 - l$ action configurations that don't result in maximal reward in the first $l^2 - l$ steps and receive maximal reward at the $l^2 - l + 1$ step. Therefore, we have $\expectation[T^{\text{full}}] = \gO(l^2 - l + 1) = \gO(l^2)$, which concludes the proof.
\end{proof}

\begin{table*}[t]
\centering
\begin{tabular}{l|cc}
\specialrule{.15em}{.05em}{.05em}
          &  CMAE with QMIX  &  \makecell{QMIX + bonus \\ Weighted QMIX + bouns}\\
\hline
\hline
Batch size &        32     & 32\\
Discounted factor & 0.99            &   0.99   \\
Critic learning rate &      0.0005       &   0.0005   \\
Agent learning rate &       0.0005      &    0.0005  \\
Optimizer &    RMSProp         &   RMSProp   \\
Replay buffer size &     5000        &  5000     \\
Epsilon anneal step &     50000        &    \{50000, 1M\}  \\
Exploration bonus coefficient  &    N.A.        &     \{1, 10, 50\} \\
Goal bonus ($\hat{r}$)  & \{0.01, 0.1, 1\} & N.A. \\
\specialrule{.15em}{.05em}{.05em}
\end{tabular}
\caption{Hyper-parameters of CMAE and baselines for SMAC tasks. }
\label{tb:sup_hyper_smac}
\end{table*}

\section{Details Regarding MPE Environments} 
\label{sec:sup_mpe}

In this section we provide  details regarding the sparse-reward and dense-reward version of MPE tasks. We first present the sparse-reward version of MPE:
\begin{itemize}
    \item \emph{Pass-sparse}: Two agents operate within two rooms of a $30 \times 30$ grid. There is one switch in each room, the rooms are separated by a door and agents start in the same room. The door will open only when one of the switches is occupied. 
    The agents see collective positive reward and the episode terminates only when both agents changed to the other room. The task is considered solved if both agents are in the right room. 
    \item \emph{Secret-Room-sparse}: \emph{Secret-Room-sparse}  extends  \emph{Pass-sparse}. There are two agents and four rooms. One large room on the left and three small rooms on the right. There is one door between each small room and the large room. The switch in the large room controls all three doors. The switch in each small room only controls the room's door. All agents need to navigate to one of the three small rooms, \ie, target room, to receive positive reward. The grid size is $25 \times 25$.  The task is considered solved if both agents are in the target room. 
    \item \emph{Push-Box-sparse}: There are two agents and one box in a $15\times15$ grid. Agents need to push the box to the wall to receive positive reward. The box is heavy, so both agents need to push the box in the same direction at the same time to move the box.  The task is considered solved if the box is pushed to the wall. 
    \item \emph{Island-sparse}: Two agents and a wolf operate in a $10\times10$ grid. Agents get a collective reward of 
    300 when crushing the wolf. The wolf and agents have maximum energy of eight and five respectively. The energy will decrease by one when being attacked. Therefore, one agent cannot crush the wolf.  The agents need to collaborate to complete the task. The task is considered solved if the wolf's health reaches zero. 
\end{itemize}

To study the performance of CMAE and baselines in a dense-reward setting, we add `checkpoints' to guide the learning of the agents. Specifically, to add checkpoints, we draw concentric circles around a landmark, \eg, a switch, a door, a box. Each circle is a checkpoint region. Then, the first time an agent steps in each of the checkpoint regions, the agent receive an additional checkpoint reward of $+0.1$.

\begin{itemize}
    \item \emph{Pass-dense}: Similar to \emph{Pass-sparse}, but the agents see dense checkpoint rewards when they move toward the switches and the door.     {Specifically, when the door is open, agents receive up to ten checkpoint rewards when they move toward the door and the switch in the right room. }
    \item \emph{Secret-Room-dense}: Similar to \emph{Secret-Room-sparse}, but the checkpoint rewards based on the agents' distance to the door and the target room's switch are added. {Specifically, when the door is open, agents receive up to ten checkpoint rewards when they move toward the door and the switch in the target room.}

    \item \emph{Push-Box-dense}: Similar to \emph{Push-Box-sparse}, but the checkpoint rewards based on the ball's distance to the wall is added. {Specifically, agents receive up to six checkpoint rewards when they push the box toward the wall.}
    
       \item \emph{Island-dense}: Similar to \emph{Island-sparse}, but the agent receives $+1$ reward when the wolf's energy decrease.
\end{itemize}

\section{Details of SMAC environments} 
\label{sec:sup_smac}
In this section, we present  details for the sparse-reward and dense-reward versions of the SMAC tasks. We first discuss the sparse-reward version of the SMAC tasks.  
\begin{itemize} 

\item \emph{3m-sparse}: There are three marines in each team. Agents need to collaboratively take care of the three marines on the other team. Agents only see a reward of $+1$ when all enemies are taken care of. 
\item \emph{2m\_vs\_1z-sparse:} There are two marines on our team and one Zealot on the opposing team. In \emph{2m\_vs\_1z-dense}, Zealots are stronger than marines. To take care of the Zealot, the marines need to learn to  fire alternatingly so as to confuse the Zealot.  Agents only see a reward of $+1$ when all enemies are taken care of. 
 \item \emph{3s\_vs\_5z-sparse}: There are three Stalkers on our team and five Zealots on the opposing team. Because Zealots counter Stalkers, the Stalkers have to learn to force the enemies to scatter around the map and attend to them one by one.  Agents only see a reward of $+1$ when all enemies are attended to. 
\end{itemize}

The details of the dense-reward version of the SMAC tasks are as follows.

\begin{itemize}
    \item \emph{3m-dense}: This task is similar to \emph{3m-sparse}, but the reward is dense. An agent sees a reward of $+1$ when it causes damage to an enemy's health. A reward of $-1$ is received when its health decreases. All the rewards are collective. A reward of $+200$ is obtained when all enemies are taken care of.
    \item \emph{2m\_vs\_1z-dense}: Similar to \emph{2m\_vs\_1z-sparse}, but the reward is dense. The reward function is similar to  \emph{3m-dense}.
    \item \emph{3s\_vs\_5z-dense}: Similar to \emph{3s\_vs\_5z-sparse}, but the reward is dense. The reward function follows the one in the  \emph{3m-dense} task.
    
\end{itemize}
Note that for all SMAC experiments we used StarCraft version SC2.4.6.2.69232. The results for different versions are not directly comparable since the underlying dynamics  differ. Please see~\citet{smac}\footnote{https://github.com/oxwhirl/smac} for more details regarding the SMAC environment.

\section{Implementation Details}
\label{sec:sup_implement}
\subsection{Normalized Entropy Estimation}
\label{subsec:ent}
As discussed in~\secref{sec:app}, we use~\equref{eq:ent} to compute the  normalized entropy for a restricted space $\cS_k$, \ie,
$$
\eta_k = H_k/H_{\text{max}, k} 
        = -\left(\sum_{s \in \cS_k}p_k(s)\log p_k(s)\right)/\log(|\cS_k|). 
$$
Note that $|\cS_k|$ is typically unavailable even in discrete state spaces. Therefore, we use the number of current observed distinct outcomes $|\hat{\cS_k}|$ to estimate $|\cS_k|$. For instance, suppose $\cS_k$ is a one-dimensional restricted state space and we observe $\cS_k$ takes values $-1, 0, 1$. Then $|\hat{\cS_k}| = 3$ is used to estimate $|\cS_k|$ in~\equref{eq:ent}.
$|\hat{\cS_k}|$ typically gradually increases during exploration. 
In addition, for $|\hat{\cS_k}|=1$, \ie, for a constant restricted space, the normalized entropy will be set to infinity.  
\input{sup_smac_mpe_tabs}

\subsection{Architecture and Hyper-Parameters} 
\label{sec:sup_param}

We present the details of architectures and hyper-parameters of CMAE and baselines next. 

\textbf{MPE environments:} 
We combine CMAE with  Q-learning. For \emph{Pass}, \emph{Secret-room}, and \emph{Push-box}, the Q value function is represented via a table. The Q-table is initialized to zero. The update step size for exploration policies and target policies are $0.1$ and $0.05$ respectively. 
For \emph{Island} we use a DQN~\citep{dqn1, dqn2}. The Q-function is parameterized by a three-layer perceptron (MLP) with 64 hidden units per layer and ReLU activation function. The learning rate is $0.0001$ and the replay buffer size is $1M$. In all MPE tasks, the bonus $\hat{r}$ for reaching a goal is $1$, and the discount factor $\gamma$ is $0.95$.

For the baseline EITI and EDTI~\cite{eiti}, we use their default architecture and hyper-parameters. The main reason that EITI and EDTI need a lot of environment steps for convergence according to our observations: a long rollout (512 steps $\times$ 32 processes) between model updates is used. In an attempt to optimize the data efficiency of baselines,  we also study shorter rollout length, \ie, \{128, 256\}, for both EITI and EDTI. However,  we didn't observe an improvement over the default setting. Specifically,  after more than $500$M environment steps of training on \emph{Secret-Room}, EITI with 128 and 256 rollout length achieves $0.0\%$  and $54.8\%$  success rate. EDTI with 128 and 256 rollout length achieves $0.0\%$  and $59.6\%$  success rate, which is much lower than the success rate of $80\%$  achieved by using the default setting. 

\textbf{SMAC environment:} We combine CMAE with QMIX~\citep{qmix}.  Following their default setting, for both exploration and target policies, the agent is a DRQN~\cite{drqn} with a GRU~\cite{gru} recurrent layer with a $64$-dimensional hidden state. Before and after the GRU layer is a fully-connected layer of $64$ units. The mix network has $32$ units. The discount factor $\gamma$ is $0.99$. The replay memory stores the latest $5000$ episodes, and the batch size is $32$. RMSProp is used with a learning rate of $5 \cdot 10^{-4}$. The target network is updated every $100$ episodes. For goal bonus $\hat{r}$ (\algref{alg:trainexp}), we studied $\{0.01, 0.1, 1\}$ and found 0.1 to work well in most tasks. Therefore, we use $\hat{r}=0.1$ for all SMAC tasks. The hyper-parameters of CMAE with QMIX and baselines are summarized in~\tabref{tb:sup_hyper_smac}.

\section{Absolute Metric and Final Metric}
\label{sec:sup_abs_metric}
In addition to the final metric reported in~\tabref{tb:mpe} and~\tabref{tb:smac}, following~\citet{Henderson17, Colas18}, we also report the \emph{absolute metric}. Absolute metric is the best policies' average episode reward over $100$ evaluation episodes. The final metric and absolute metric of CMAE and baselines on MPE and SMAC tasks are summarized in~\tabref{tb:sup_mpe} and~\tabref{tb:sup_smac}.

\section{Additional Results on MPE Task: Island}
\label{sec:sup_island}
In addition to the MPE tasks considered in~\secref{sec:experimental}, we consider one more challenging MPE task: Island. The details of both sparse-reward and dense-reward version of Island, \ie, \emph{Island-sparse} and \emph{Island-dense} are presented in~\secref{sec:sup_mpe}. We compare CMAE to Q-learning, Q-learning with count-based exploration, EITI, and EDTI on both \emph{Island-sparse} and \emph{Island-dense}. The results are summarized in~\tabref{tb:sup_mpe}. As~\tabref{tb:sup_mpe} shows, in the sparse-reward setting, CMAE is able to achieve higher than $50\%$ success rate. In contrast, baselines struggle to solve the task. In the dense-reward setting, CMAE  performs similar to baselines. 
The training curves are shown  in~\figref{fig:sup_mpe-sparse_plot} and~\figref{fig:sup_mpe-dense_plot}. 

\input{sup_mpe_steps}
\section{Additional Results of Baselines}
\label{sec:sup_more_steps}
Following the setting of EITI and EDTI~\cite{eiti}, we train both baselines for 500M environment steps.
On \emph{Pass-sparse}, \emph{Secret-Room-sparse}, and \emph{Push-Box-sparse}, we observe that EITI and EDTI~\citep{eiti} need more than 300M steps to achieve an $80\%$ success rate.  
In contrast, CMAE achieves a $100\%$ success rate  within $3$M environment steps. On \emph{Island-sparse}, EITI and EDTI need more than $3$M environment steps to achieve a $20\%$ success rate while CMAE needs  less than $8$M environment steps to achieve the same success rate. 
The results are summarized in~\tabref{tb:sup_mpe_step}.

\section{Additional Training Curves}
\label{sec:sup_curve}
The training curves of CMAE and baselines on both sparse-reward and dense-reward MPE tasks are shown in~\figref{fig:sup_mpe-sparse_plot} and~\figref{fig:sup_mpe-dense_plot}. The training curves of CMAE and baselines on both sparse-reward and dense-reward SMAC tasks are shown in~\figref{fig:sup_smac-sparse_plot} and~\figref{fig:sup_smac-dense_plot}. 
{As shown in~\figref{fig:sup_mpe-sparse_plot},~\figref{fig:sup_mpe-dense_plot},~\figref{fig:sup_smac-sparse_plot}, and ~\figref{fig:sup_smac-dense_plot}, in challenging sparse-reward tasks, CMAE consistently achieves higher success rate than baselines. In dense-reward tasks, CMAE has similar performance to baselines. }

\begin{figure*}[t]

\centering
\renewcommand{\arraystretch}{0}
\begin{tabular}{ccc}
\includegraphics[width=0.33\textwidth]{fig/Room.pdf}
&
\hspace{-0.6cm}
\includegraphics[width=0.33\textwidth]{fig/Secret_Room.pdf}
&
\hspace{-0.6cm}
\includegraphics[width=0.33\textwidth]{fig/Push_Box.pdf}
\\ 
\includegraphics[width=0.33\textwidth]{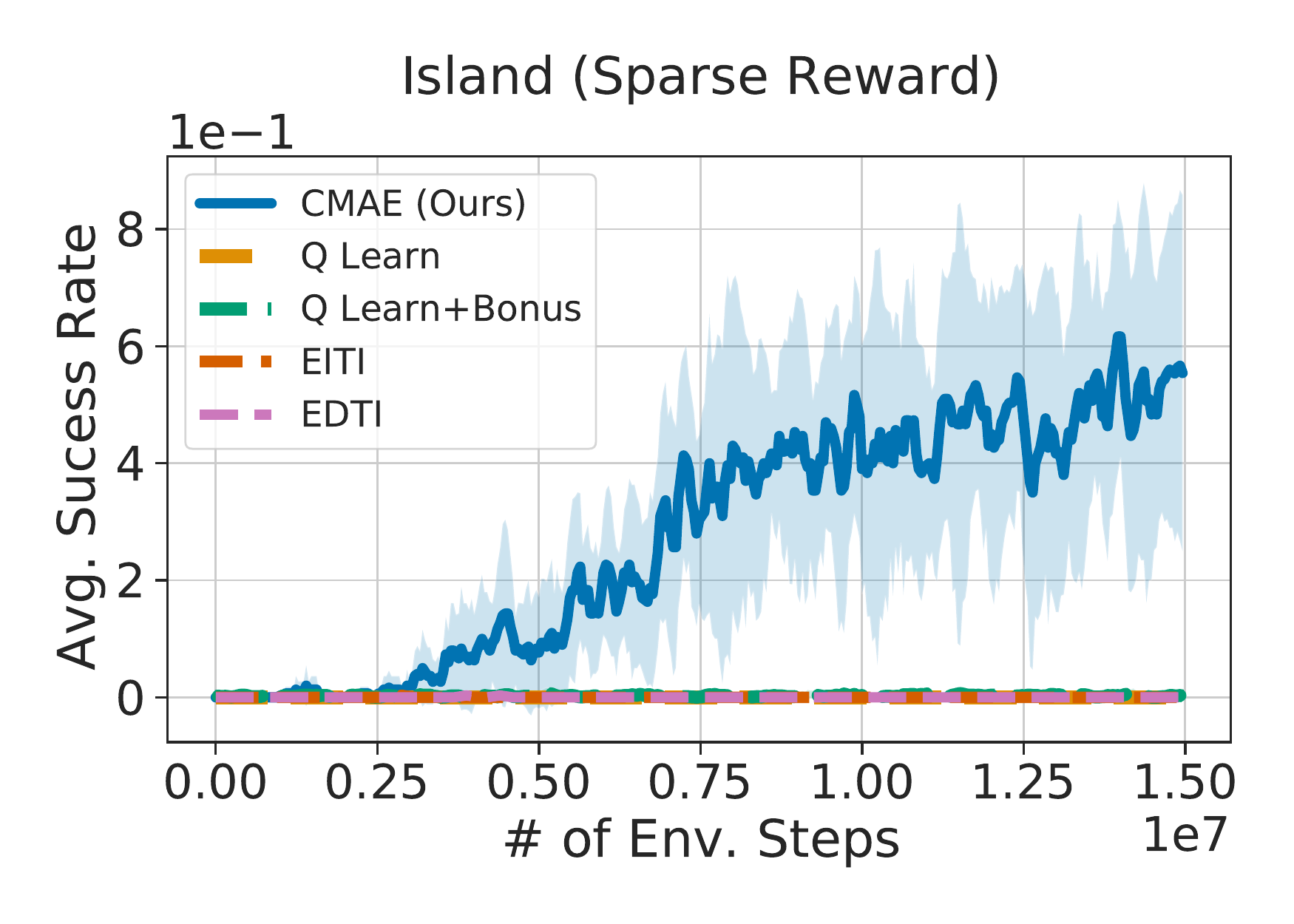}

\end{tabular}%
\caption{Training curves on sparse-reward MPE tasks.} 
\label{fig:sup_mpe-sparse_plot}
\end{figure*}

\begin{figure*}[t]

\centering
\renewcommand{\arraystretch}{0}
\begin{tabular}{ccc}
\includegraphics[width=0.33\textwidth]{fig/Room_ckpt.pdf}
&
\hspace{-0.6cm}
\includegraphics[width=0.33\textwidth]{fig/Secret_Room_ckpt.pdf}
&
\hspace{-0.6cm}
\includegraphics[width=0.33\textwidth]{fig/Push_Box_ckpt.pdf}
\\ 
\includegraphics[width=0.33\textwidth]{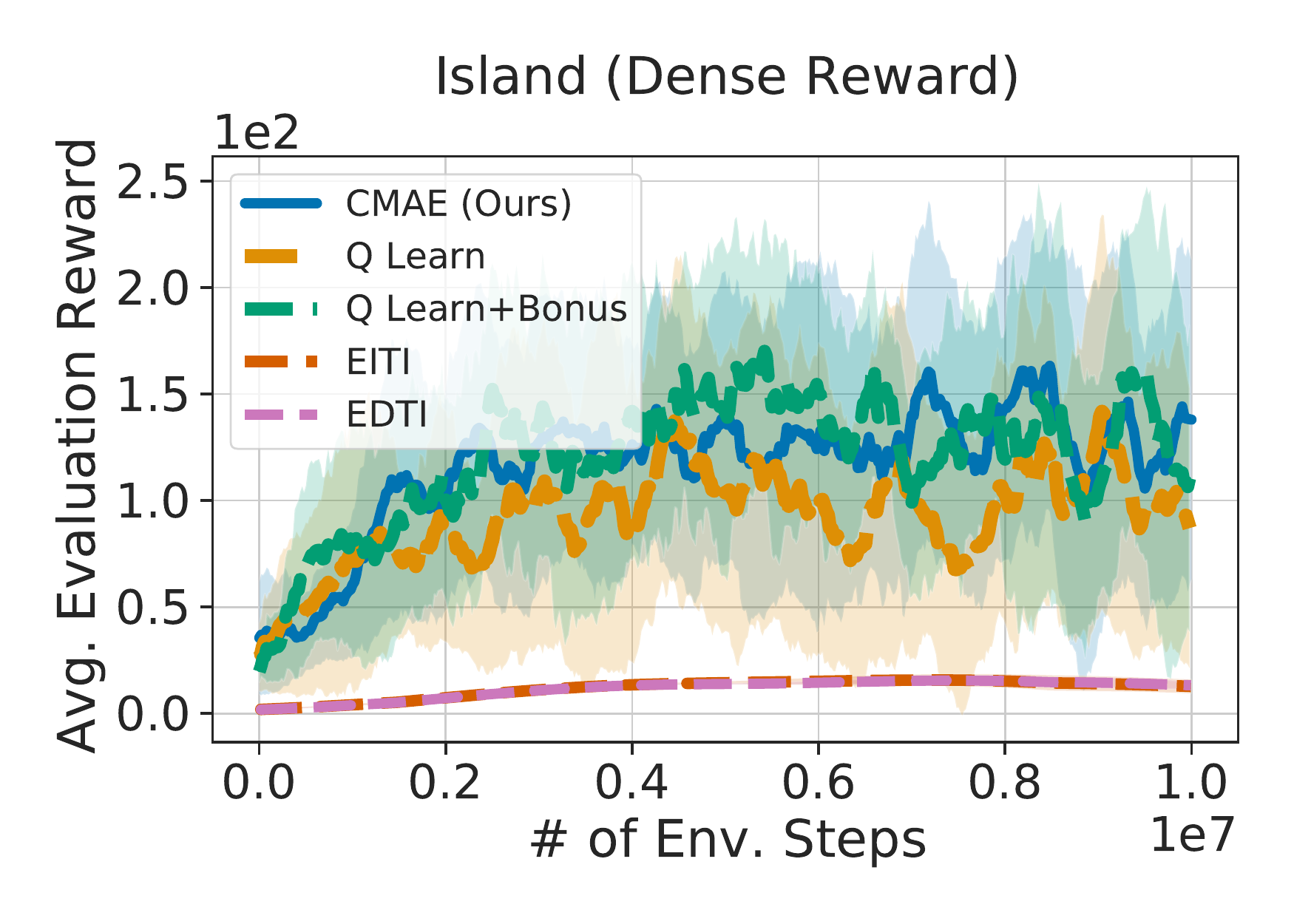}

\end{tabular}%
\caption{Training curves on dense-reward MPE tasks.} 
\label{fig:sup_mpe-dense_plot}
\end{figure*}

\begin{figure*}[t]

\centering
\begin{tabular}{cccc}
\includegraphics[width=0.33\textwidth, height=4cm]{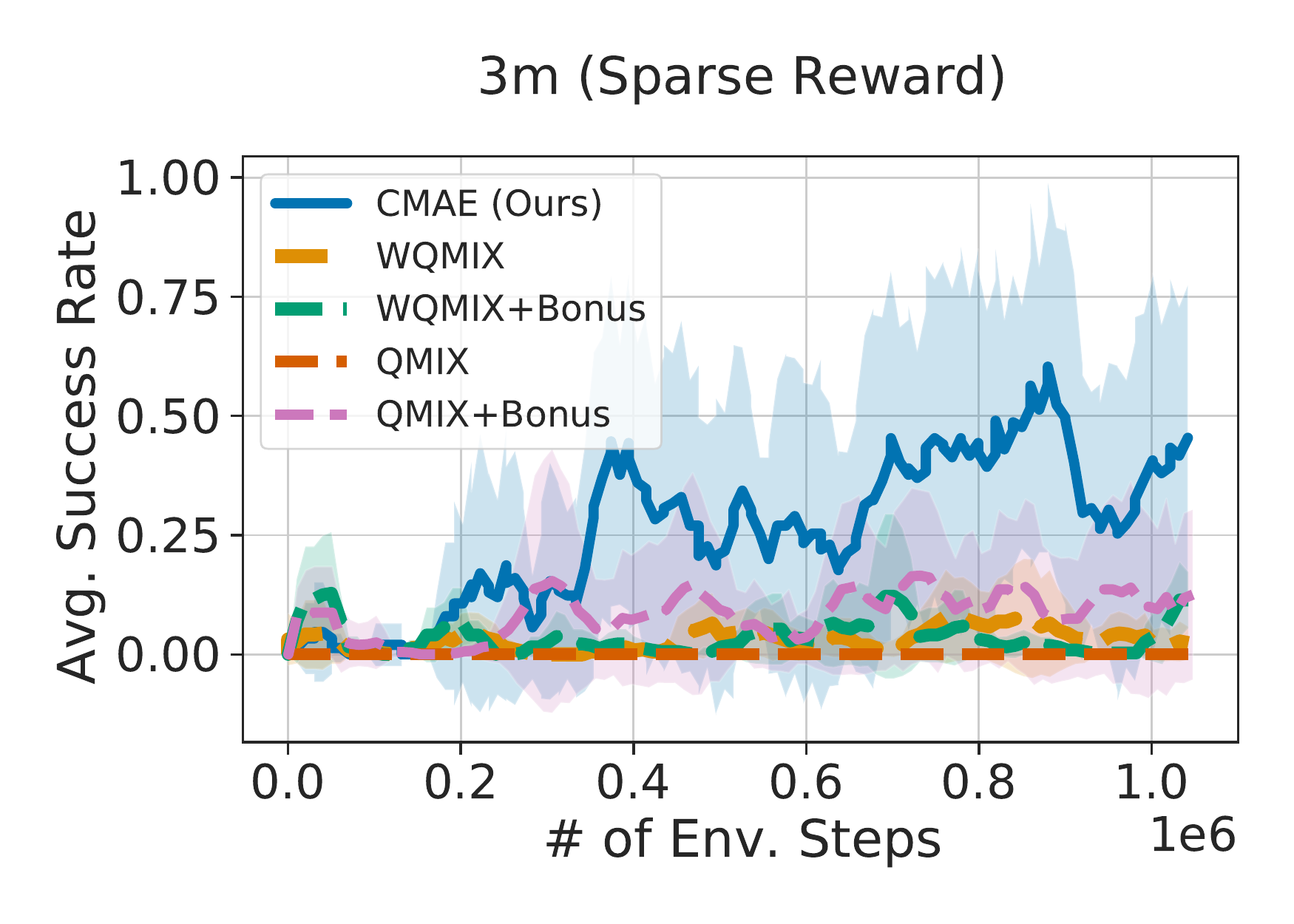}
&
\hspace{-0.6cm}
\includegraphics[width=0.33\textwidth, height=4cm]{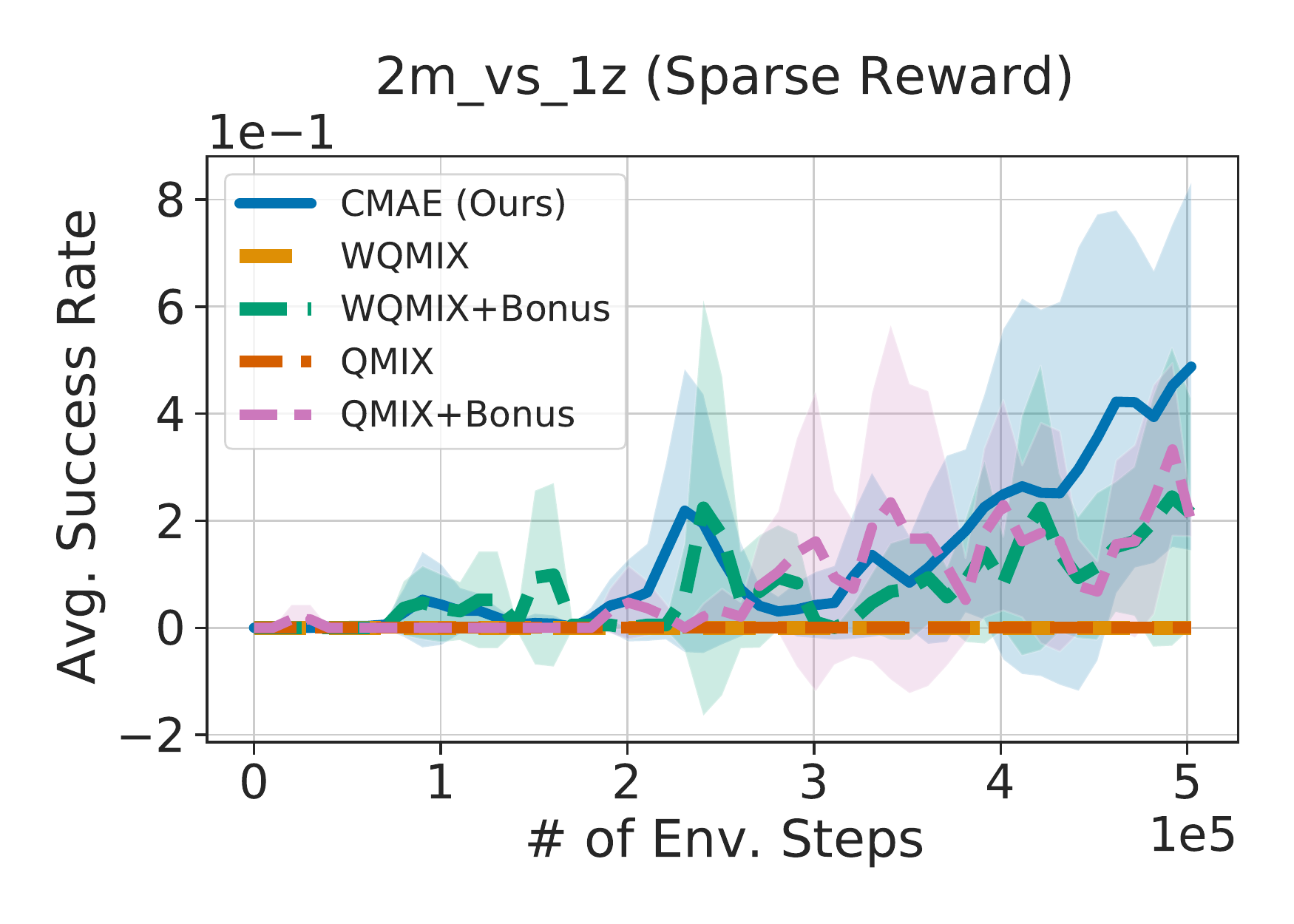}
&
\hspace{-0.6cm}
\includegraphics[width=0.33\textwidth, height=4cm]{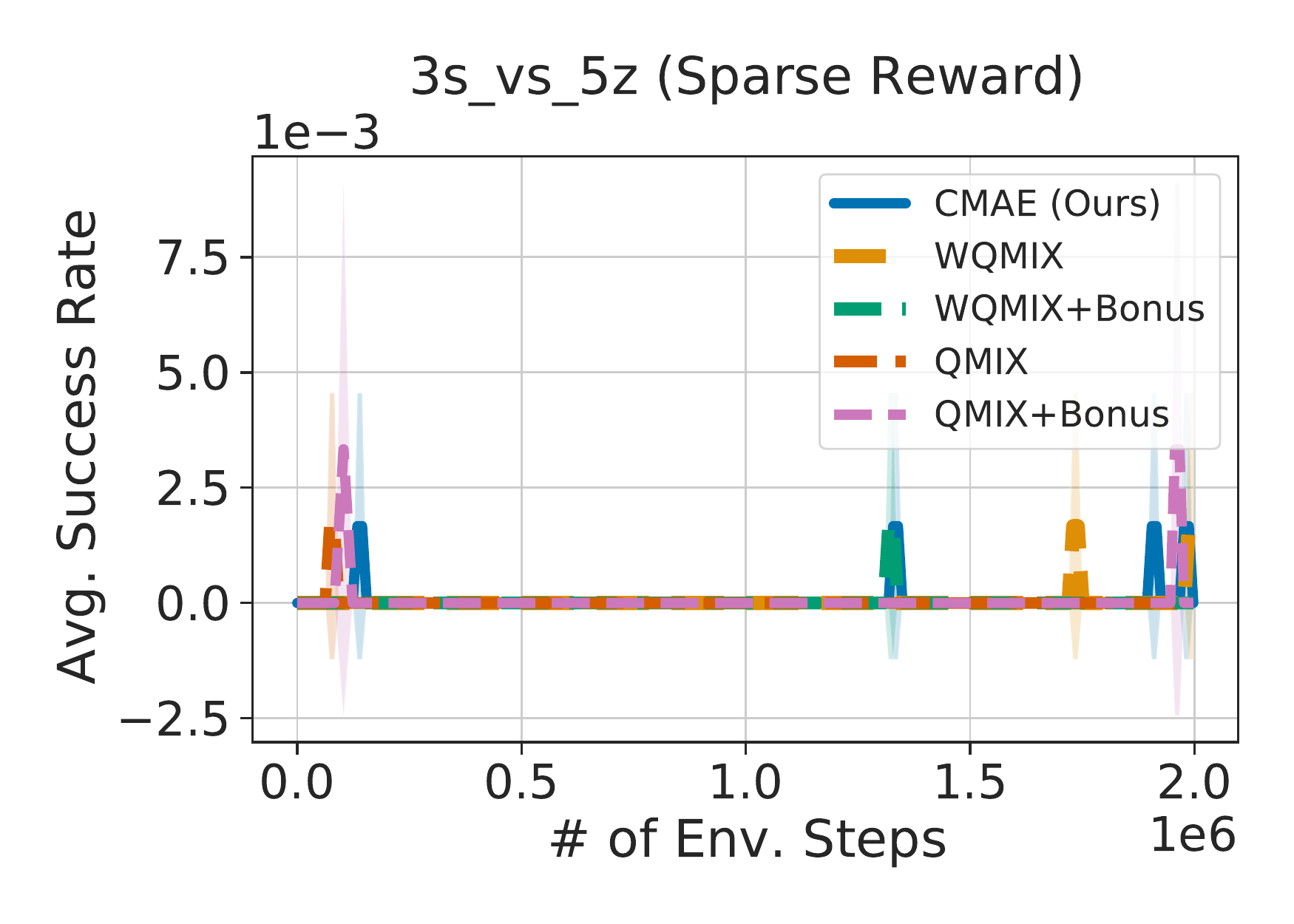}

\end{tabular}%
\caption{Training curves on sparse-reward SMAC tasks.} 
\label{fig:sup_smac-sparse_plot}
\end{figure*}

\begin{figure*}[t]

\centering
\begin{tabular}{cccc}
\includegraphics[width=0.33\textwidth, height=4cm]{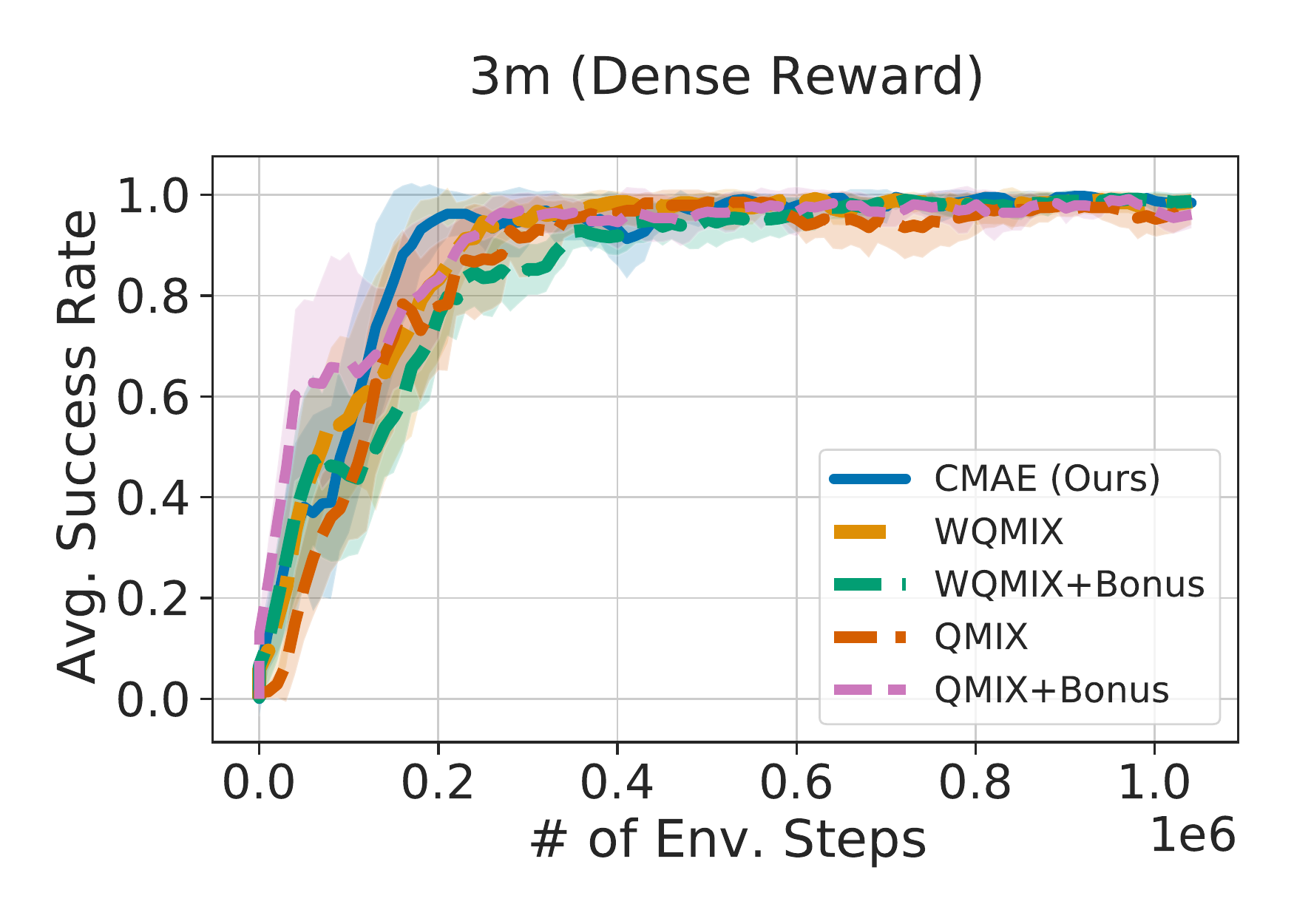}
&
\hspace{-0.6cm}
\includegraphics[width=0.33\textwidth, height=4cm]{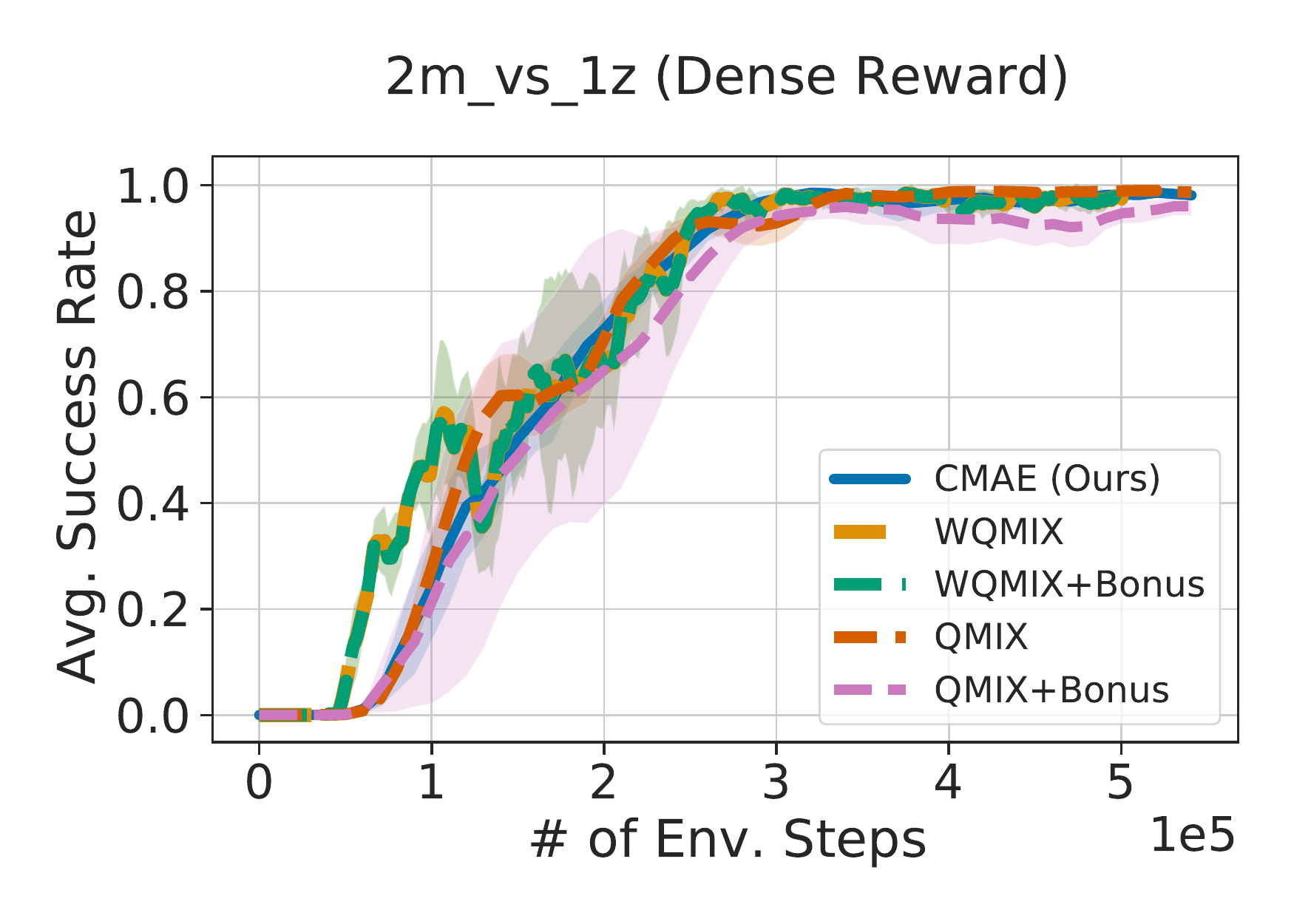}
&
\hspace{-0.6cm}
\includegraphics[width=0.33\textwidth, height=4cm]{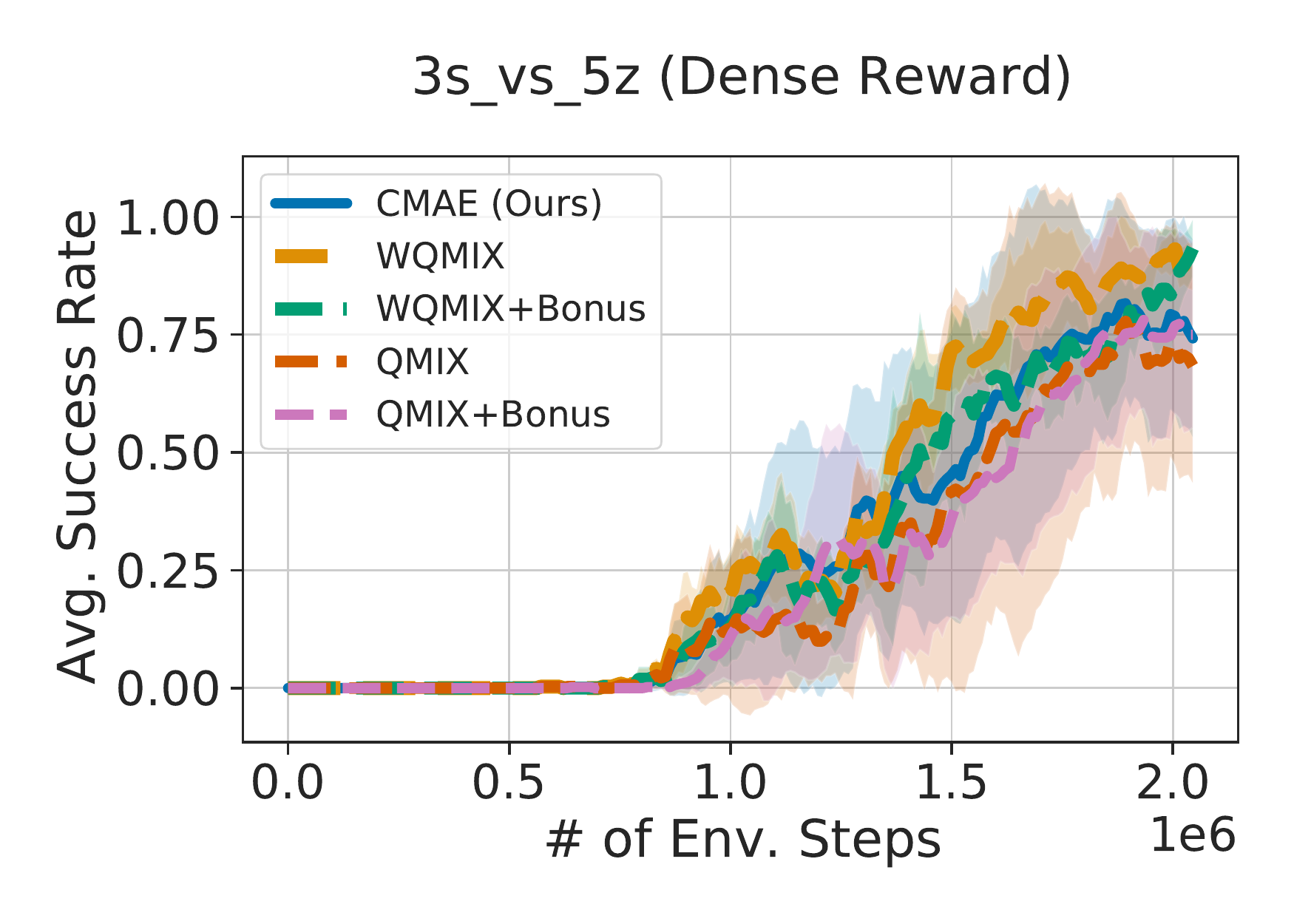}

\end{tabular}%
\caption{Training curves on dense-reward SMAC tasks.} 
\label{fig:sup_smac-dense_plot}
\end{figure*}

%% file: sup_smac_mpe_tabs.tex

\begin{table*}[t]
\centering
\setlength{\tabcolsep}{4pt}
\begin{tabular}{l|c|ccccc}
\specialrule{.15em}{.05em}{.05em}
          & & CMAE (Ours)   & Q-learning & \makecell{Q-learning + Bonus} &  EITI & EDTI \\
\hline
\hline
\multirow{2}{*}{Pass-sparse} &Final &      \textbf{1.00$\pm$0.00}    &    0.00$\pm$0.00 & 0.00$\pm$0.00 & 0.00$\pm$0.00 & 0.00$\pm$0.00 \\
                             &Absolute & \textbf{1.00$\pm$0.00}    &    0.00$\pm$0.00 & 0.00$\pm$0.00 & 0.00$\pm$0.00 & 0.00$\pm$0.00\\
\hline
\multirow{2}{*}{Secret-Room-sparse} &Final &   \textbf{1.00$\pm$0.00}    &    0.00$\pm$0.00 & 0.00$\pm$0.00 & 0.00$\pm$0.00 & 0.00$\pm$0.00 \\
&Absolute &   \textbf{1.00$\pm$0.00}    &    0.00$\pm$0.00 & 0.00$\pm$0.00 & 0.00$\pm$0.00 & 0.00$\pm$0.00 \\
\hline
\multirow{2}{*}{Push-Box-sparse}&Final  &      \textbf{1.00$\pm$0.00}    &    0.00$\pm$0.00 & 0.00$\pm$0.00 & 0.00$\pm$0.00 & 0.00$\pm$0.00 \\
&Absolute  &      \textbf{1.00$\pm$0.00}    &    0.00$\pm$0.00 & 0.00$\pm$0.00 & 0.00$\pm$0.00 & 0.00$\pm$0.00 \\
\hline
\multirow{2}{*}{Island-sparse}&Final  &      \textbf{0.55$\pm$0.30}    &    0.00$\pm$0.00 & 0.00$\pm$0.00 & 0.00$\pm$0.00 & 0.00$\pm$0.00 \\
&Absolute  &      \textbf{0.61$\pm$0.23}    &    0.00$\pm$0.00 & 0.01$\pm$0.01 & 0.00$\pm$0.00 & 0.00$\pm$0.00 \\
\hline
\hline
\multirow{2}{*}{Pass-dense} &Final &        \textbf{5.00$\pm$0.00}    &    1.25$\pm$0.02  &      1.42$\pm$0.14 & 0.00$\pm$0.00 &0.18$\pm$0.01  \\
&Absolute &        \textbf{5.00$\pm$0.00}    &    1.30$\pm$0.03  &      1.46$\pm$0.08 & 0.00$\pm$0.00 &0.20$\pm$0.01  \\
\hline
\multirow{2}{*}{Secret-Room-dense} &Final &        \textbf{4.00$\pm$0.57}  &   1.62$\pm$0.16   &     1.53$\pm$0.04 & 0.00$\pm$0.00 &0.00$\pm$0.00  \\
&Absolute &        \textbf{4.00$\pm$0.57}    &    1.63$\pm$0.03  &      1.57$\pm$0.06 & 0.00$\pm$0.00 &0.00$\pm$0.00  \\
\hline
\multirow{2}{*}{Push-Box-dense} &Final &        1.38$\pm$0.21  &    \textbf{1.58$\pm$0.14}  & 1.55$\pm$0.04  & 0.10$\pm$0.01 & 0.05$\pm$0.03  \\
&Absolute &        1.38$\pm$0.21    &    \textbf{1.59$\pm$0.04}  &      1.55$\pm$0.04 & 0.00$\pm$0.00 &0.18$\pm$0.01  \\
\hline
\multirow{2}{*}{Island-dense} &Final &        \textbf{138.00$\pm$74.70}  &    87.03$\pm$65.80  & 110.36$\pm$71.99  & 11.18$\pm$0.62 & 10.45$\pm$0.61  \\
&Absolute &        163.25$\pm$68.50    &    141.60$\pm$92.53  &      \textbf{170.14$\pm$62.10} & 16.84$\pm$0.65 &16.42$\pm$0.86  \\

\specialrule{.15em}{.05em}{.05em}
\end{tabular}
\caption{\emph{Final metric} and \emph{absolute metric} of CMAE and baselines on sparse-reward and dense-reward MPE tasks.}
\label{tb:sup_mpe}
\end{table*}

\begin{table*}[t]
\centering
\setlength{\tabcolsep}{4pt}
\begin{tabular}{l|c|ccccc}
\specialrule{.15em}{.05em}{.05em}
          & & CMAE (Ours)  &  Weighted QMIX &  Weighted QMIX + Bonus & QMIX & QMIX + Bonus \\
\hline
\hline
\multirow{2}{*}{3m-sparse} &   Final &     \textbf{47.7$\pm$35.1}     & 2.7$\pm$5.1 &11.5$\pm$8.6 &    0.0$\pm$0.0  &      11.7$\pm$16.9\\
&   Absolute &     \textbf{62.0$\pm$41.0}     & 8.1$\pm$4.5 &15.6$\pm$7.3 &    0.0$\pm$0.0  &      22.8$\pm$18.4\\
\hline
\multirow{2}{*}{2m\_vs\_1z-sparse} &   Final &     \textbf{44.3$\pm$20.8}   &0.0$\pm$0.0 &19.4$\pm$18.1 &   0.0$\pm$0.0   &     19.8$\pm$14.1\\
&   Absolute &     \textbf{47.7$\pm$35.1}     & 0.0$\pm$0.0 &23.9$\pm$16.7 &    0.0$\pm$0.0  &      30.3$\pm$26.7\\
\hline
\multirow{2}{*}{3s\_vs\_5z-sparse} &   Final &     0.0$\pm$0.0  &    0.0$\pm$0.0  & 0.0$\pm$0.0  & 0.0$\pm$0.0 &0.0$\pm$0.0\\
&   Absolute &     0.0$\pm$0.0  &    0.0$\pm$0.0  & 0.0$\pm$0.0  & 0.0$\pm$0.0 &0.0$\pm$0.0\\
\hline
\hline
\multirow{2}{*}{3m-dense} &   Final &     98.7$\pm$1.7    &    98.3$\pm$2.5  &      \textbf{98.9$\pm$1.7} & 97.9$\pm$3.6 &97.3$\pm$3.0\\
&   Absolute &     99.3$\pm$1.8    & 98.8$\pm$0.3 &99.0$\pm$0.3 &    \textbf{99.4$\pm$2.1} &      98.5$\pm$1.2\\
\hline
\multirow{2}{*}{2m\_vs\_1z-dense} &   Final &     98.2$\pm$0.1  &   \textbf{98.5$\pm$0.1}   &     96.0$\pm$1.8 & 97.1$\pm$2.4 &95.8$\pm$1.7\\
&   Absolute &     98.7$\pm$0.4     & 98.6$\pm$1.6 &99.1$\pm$0.9 &    \textbf{99.1$\pm$0.6}  &      96.0$\pm$1.6\\
\hline
\multirow{2}{*}{3s\_vs\_5z-dense} &   Final &     81.3$\pm$16.1  &    92.2$\pm$6.6  & \textbf{95.3$\pm$2.2}  & 75.0$\pm$17.6 &78.1$\pm$24.4\\
&   Absolute &     85.4$\pm$22.6  &    \textbf{95.4$\pm$4.4}  & 95.4$\pm$3.2  & 76.5$\pm$24.3 &79.1$\pm$14.2\\
\specialrule{.15em}{.05em}{.05em}
\end{tabular}
\caption{\emph{Final metric} and \emph{absolute metric} of success rate ($\%$) of CMAE and baselines on sparse-reward and dense-reward SMAC tasks. }
\label{tb:sup_smac}
\end{table*}

%% file: sup_mpe_steps.tex
\begin{table*}[t]
\centering
\begin{tabular}{l|lll}
\specialrule{.15em}{.05em}{.05em}
         Task (target success rate) &  CMAE (Ours)  &  EITI & EDTI\\
\hline
\hline
Pass-sparse ($80\%$)&        \textbf{2.43M{$\pm$0.10M}}    &    384M{$\pm$1.2M}  &      381M{$\pm$2.8M}   \\
Secret-Room-sparse ($80\%$)&   \textbf{2.35M{$\pm$0.05M}}  &   448M{$\pm$10.0M}   &     382M{$\pm$9.4M}   \\
Push-Box-sparse ($10\%$) &      \textbf{0.47M{$\pm$0.04M}}  &    307M{$\pm$2.3M}  & 160M{$\pm$12.1M}  \\
Push-Box-sparse ($80\%$) &      \textbf{2.26M{$\pm$0.02M}}  &    307M{$\pm$3.9M}  & 160M{$\pm$8.2M}   \\
Island-sparse ($20\%$)&        \textbf{7.50M$\pm$0.12M}  &   480M$\pm$5.2M   &  322M$\pm$1.4M \\
Island-sparse ($50\%$)&        \textbf{13.9M$\pm$0.21M}  &   $>$ 500M   &  $>$ 500M\\

\specialrule{.15em}{.05em}{.05em}
\end{tabular}
\includegraphics[width=0.5\textwidth]{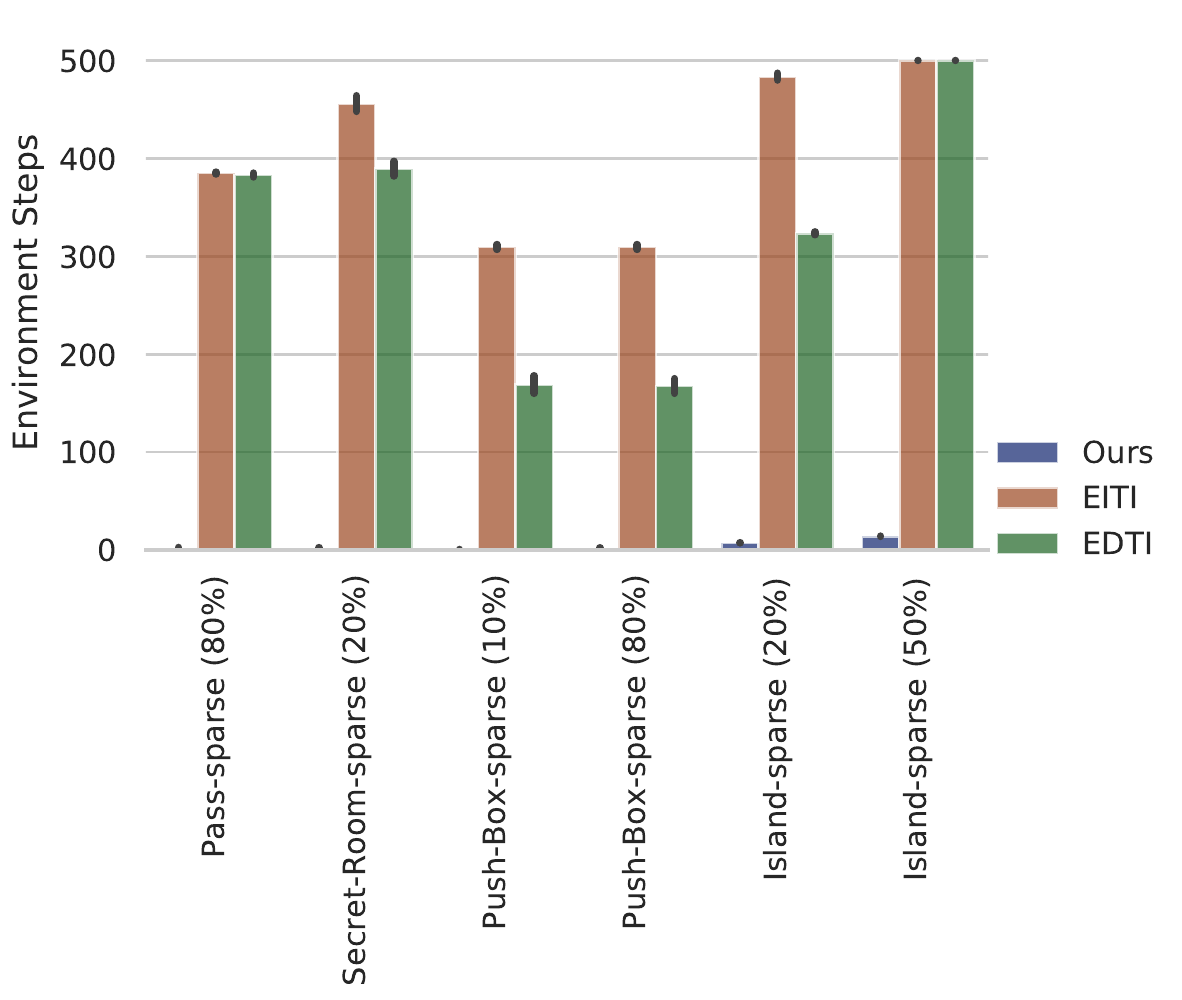}
\caption{Environment steps required to achieve the indicated target success rate on \emph{Pass-sparse}, \emph{Secret-Room-sparse}, \emph{Push-Box-sparse}, and \emph{Island-sparse} environments.}
\label{tb:sup_mpe_step}
\end{table*}

%% file: main.bbl
\begin{thebibliography}{62}
\providecommand{\natexlab}[1]{#1}
\providecommand{\url}[1]{\texttt{#1}}
\expandafter\ifx\csname urlstyle\endcsname\relax
  \providecommand{\doi}[1]{doi: #1}\else
  \providecommand{\doi}{doi: \begingroup \urlstyle{rm}\Url}\fi

\bibitem[Andrychowicz et~al.(2017)Andrychowicz, Wolski, Ray, Schneider, Fong,
  Welinder, McGrew, Tobin, Abbeel, and Zaremba]{her}
Andrychowicz, M., Wolski, F., Ray, A., Schneider, J., Fong, R., Welinder, P.,
  McGrew, B., Tobin, J., Abbeel, P., and Zaremba, W.
\newblock Hindsight experience replay.
\newblock In \emph{Proc. NeurIPS}, 2017.

\bibitem[Bazzan(2008)]{Bazzan08}
Bazzan, A. L.~C.
\newblock Opportunities for multiagent systems and multiagent reinforcement
  learning in traffic control.
\newblock In \emph{Proc. AAMAS}, 2008.

\bibitem[Bellemare et~al.(2016)Bellemare, Srinivasan, Ostrovski, Schaul,
  Saxton, and Munos]{Bellemare17}
Bellemare, M.~G., Srinivasan, S., Ostrovski, G., Schaul, T., Saxton, D., and
  Munos, R.
\newblock Unifying count-based exploration and intrinsic motivation.
\newblock In \emph{Proc. NeurIPS}, 2016.

\bibitem[Burda et~al.(2019)Burda, Edwards, Storkey, and Klimov]{Burda19}
Burda, Y., Edwards, H., Storkey, A., and Klimov, O.
\newblock Exploration by random network distillation.
\newblock In \emph{Proc. ICLR}, 2019.

\bibitem[Christianos et~al.(2020)Christianos, Schafer, and
  Albrecht]{Christianos20}
Christianos, F., Schafer, L., and Albrecht, S.~V.
\newblock Shared experience actor-critic for multi-agent reinforcement
  learning.
\newblock In \emph{Proc. NeurIPS}, 2020.

\bibitem[Chung et~al.(2014)Chung, Gulcehre, Cho, and Bengio]{gru}
Chung, J., Gulcehre, C., Cho, K., and Bengio, Y.
\newblock Empirical evaluation of gated recurrent neural networks on sequence
  modeling.
\newblock In \emph{arXiv.}, 2014.

\bibitem[Colas et~al.(2018)Colas, Sigaud, and Oudeyer]{Colas18}
Colas, C., Sigaud, O., and Oudeyer, P.-Y.
\newblock {GEP-PG}: Decoupling exploration and exploitation in deep
  reinforcement learning algorithms.
\newblock In \emph{Proc. ICML}, 2018.

\bibitem[Dimakopoulou \& Roy(2018)Dimakopoulou and Roy]{concurrent1}
Dimakopoulou, M. and Roy, B.~V.
\newblock Coordinated exploration in concurrent reinforcement learning.
\newblock In \emph{Proc. ICML}, 2018.

\bibitem[Ding et~al.(2020)Ding, Huang, and Lu]{Ding20}
Ding, Z., Huang, T., and Lu, Z.
\newblock Learning individually inferred communication for multi-agent
  cooperation.
\newblock In \emph{Proc. NeurIPS}, 2020.

\bibitem[Foerster et~al.(2017)Foerster, Nardelli, Farquhar, Afouras, Torr,
  Kohli, and Whiteson]{Foerster17}
Foerster, J., Nardelli, N., Farquhar, G., Afouras, T., Torr, P. H.~S., Kohli,
  P., and Whiteson, S.
\newblock Stabilising experience replay for deep multi-agent reinforcement
  learning.
\newblock In \emph{Proc. ICML}, 2017.

\bibitem[Foerster et~al.(2018)Foerster, Farquhar, Afouras, Nardelli, and
  Whiteson]{coma}
Foerster, J., Farquhar, G., Afouras, T., Nardelli, N., and Whiteson, S.
\newblock Counterfactual multi-agent policy gradients.
\newblock In \emph{Proc. AAAI}, 2018.

\bibitem[Foerster et~al.(2016)Foerster, Assael, de~Freitas, and
  Whiteson]{Foerster16}
Foerster, J.~N., Assael, Y.~M., de~Freitas, N., and Whiteson, S.
\newblock Learning to communicate with deep multi-agent reinforcement learning.
\newblock In \emph{Proc. ICML}, 2016.

\bibitem[Fu et~al.(2017)Fu, Co-Reyes, and Levine]{Fu17}
Fu, J., Co-Reyes, J.~D., and Levine, S.
\newblock Ex2: Exploration with exemplar models for deep reinforcement
  learning.
\newblock In \emph{Proc. NeurIPS}, 2017.

\bibitem[Hausknecht \& Stone(2015)Hausknecht and Stone]{drqn}
Hausknecht, M. and Stone, P.
\newblock Deep recurrent q-learning for partially observable mdps.
\newblock In \emph{arXiv.}, 2015.

\bibitem[Henderson et~al.(2017)Henderson, Islam, Bachman, Pineau, Precup, and
  Meger]{Henderson17}
Henderson, P., Islam, R., Bachman, P., Pineau, J., Precup, D., and Meger, D.
\newblock Deep reinforcement learning that matters.
\newblock In \emph{Proc. AAAI}, 2017.

\bibitem[Hong et~al.(2018)Hong, Shann, Su, Chang, and Lee]{Hong18}
Hong, Z.-W., Shann, T.-Y., Su, S.-Y., Chang, Y.-H., and Lee, C.-Y.
\newblock Diversity-driven exploration strategy for deep reinforcement
  learning.
\newblock In \emph{Proc. NeurIPS}, 2018.

\bibitem[Hu et~al.(2021)Hu, Zhu, Chang, and Liang]{updet}
Hu, S., Zhu, F., Chang, X., and Liang, X.
\newblock Updet: Universal multi-agent rl via policy decoupling with
  transformers.
\newblock In \emph{Proc. ICLR}, 2021.

\bibitem[H\"{u}ttenrauch et~al.(2019)H\"{u}ttenrauch, {\v{S}}o{\v{s}}i{{\'c}},
  and Neumann]{swarm}
H\"{u}ttenrauch, M., {\v{S}}o{\v{s}}i{{\'c}}, A., and Neumann, G.
\newblock Deep reinforcement learning for swarm systems.
\newblock \emph{JMLR}, 2019.

\bibitem[Inala et~al.(2020)Inala, Yang, Paulos, Pu, Bastani, Kumar, Rinard, and
  Solar-Lezama]{Priya20}
Inala, J.~P., Yang, Y., Paulos, J., Pu, Y., Bastani, O., Kumar, V., Rinard, M.,
  and Solar-Lezama, A.
\newblock Neurosymbolic transformers for multi-agent communication.
\newblock In \emph{Proc. NeurIPS}, 2020.

\bibitem[Jain et~al.(2019)Jain, Weihs, Kolve, Rastegari, Lazebnik, Farhadi,
  Schwing, and Kembhavi]{Jain19}
Jain, U., Weihs, L., Kolve, E., Rastegari, M., Lazebnik, S., Farhadi, A.,
  Schwing, A., and Kembhavi, A.
\newblock Two body problem: Collaborative visual task completion.
\newblock In \emph{Proc. CVPR}, 2019.

\bibitem[Jain et~al.(2020)Jain, Weihs, Kolve, Farhadi, Lazebnik, Kembhavi, and
  Schwing]{Jain20}
Jain, U., Weihs, L., Kolve, E., Farhadi, A., Lazebnik, S., Kembhavi, A., and
  Schwing, A.~G.
\newblock A cordial sync: Going beyond marginal policies for multi-agent
  embodied tasks.
\newblock In \emph{Proc. ECCV}, 2020.

\bibitem[Jain et~al.(2021)Jain, Liu, Lazebnik, Kembhavi, Weihs, and
  Schwing]{Jain21}
Jain, U., Liu, I.-J., Lazebnik, S., Kembhavi, A., Weihs, L., and Schwing, A.
\newblock Gridtopix: Training embodied agents with minimal supervision.
\newblock In \emph{arXiv.}, 2021.

\bibitem[Jiang \& Lu(2018)Jiang and Lu]{Jiang18}
Jiang, J. and Lu, Z.
\newblock Learning attentional communication for multi-agent cooperation.
\newblock In \emph{Proc. NeurIPS}, 2018.

\bibitem[Lee et~al.(2020)Lee, Eysenbach, Parisotto, Xing, Levine, and
  Salakhutdinov]{Lee20}
Lee, L., Eysenbach, B., Parisotto, E., Xing, E., Levine, S., and Salakhutdinov,
  R.
\newblock Efficient exploration via state marginal matching.
\newblock In \emph{arXiv.}, 2020.

\bibitem[Li et~al.(2019)Li, Liu, Yuan, Chen, Schwing, and Huang]{iswitch}
Li, Y., Liu, I.-J., Yuan, Y., Chen, D., Schwing, A., and Huang, J.
\newblock Accelerating distributed reinforcement learning with in-switch
  computing.
\newblock In \emph{Proc. ISCA}, 2019.

\bibitem[Lillicrap et~al.(2016)Lillicrap, Hunt, Pritzel, Heess, Erez, Tassa,
  Silver, and Wierstra]{ddpg}
Lillicrap, T.~P., Hunt, J.~J., Pritzel, A., Heess, N., Erez, T., Tassa, Y.,
  Silver, D., and Wierstra, D.
\newblock Continuous control with deep reinforcement learning.
\newblock In \emph{Proc. ICLR}, 2016.

\bibitem[Liu et~al.(2019{\natexlab{a}})Liu, Peng, and Schwing]{kf}
Liu, I.-J., Peng, J., and Schwing, A.~G.
\newblock Knowledge flow: Improve upon your teachers.
\newblock In \emph{Proc. ICLR}, 2019{\natexlab{a}}.

\bibitem[Liu et~al.(2019{\natexlab{b}})Liu, Yeh, and Schwing]{pic}
Liu, I.-J., Yeh, R.~A., and Schwing, A.~G.
\newblock {PIC:} permutation invariant critic for multi-agent deep
  reinforcement learning.
\newblock In \emph{Proc. CoRL}, 2019{\natexlab{b}}.

\bibitem[Liu et~al.(2020)Liu, Yeh, and Schwing]{hts-rl}
Liu, I.-J., Yeh, R.~A., and Schwing, A.~G.
\newblock High-throughput synchronous deep rl.
\newblock In \emph{Proc. NeurIPS}, 2020.

\bibitem[Lowe et~al.(2017)Lowe, Wu, Tamar, Harb, Abbeel, and Mordatch]{maddpg}
Lowe, R., Wu, Y., Tamar, A., Harb, J., Abbeel, P., and Mordatch, I.
\newblock Multi-agent actor-critic for mixed cooperative-competitive
  environments.
\newblock In \emph{Proc. NeurIPS}, 2017.

\bibitem[Mahajan et~al.(2019)Mahajan, Rashid, Samvelyan, and Whiteson]{maven}
Mahajan, A., Rashid, T., Samvelyan, M., and Whiteson, S.
\newblock {MAVEN:} multi-agent variational exploration.
\newblock In \emph{Proc. NeurIPS}, 2019.

\bibitem[Matignon et~al.(2012)Matignon, j.~Laurent, and fort piat]{Matignon12}
Matignon, L., j.~Laurent, G., and fort piat, N.~L.
\newblock Independent reinforcement learners in cooperative markov games: a
  survey regarding coordination problems.
\newblock \emph{The Knowledge Engineering Review,}, 2012.

\bibitem[Mnih et~al.(2013)Mnih, Kavukcuoglu, Silver, Graves, Antonoglou,
  Wierstra, and Riedmiller]{dqn1}
Mnih, V., Kavukcuoglu, K., Silver, D., Graves, A., Antonoglou, I., Wierstra,
  D., and Riedmiller, M.
\newblock Playing atari with deep reinforcement learning.
\newblock In \emph{NeurIPS Deep Learning Workshop}, 2013.

\bibitem[Mnih et~al.(2015)Mnih, Kavukcuoglu, Silver, Rusu, Veness, Bellemare,
  Graves, Riedmiller, Fidjeland, Ostrovski, Petersen, Beattie, Sadik,
  Antonoglou, King, Kumaran, Wierstra, Legg, and Hassabis]{dqn2}
Mnih, V., Kavukcuoglu, K., Silver, D., Rusu, A.~A., Veness, J., Bellemare,
  M.~G., Graves, A., Riedmiller, M., Fidjeland, A.~K., Ostrovski, G., Petersen,
  S., Beattie, C., Sadik, A., Antonoglou, I., King, H., Kumaran, D., Wierstra,
  D., Legg, S., and Hassabis, D.
\newblock Human-level control through deep reinforcement learning.
\newblock \emph{Nature}, 2015.

\bibitem[Myerson(2013)]{Myerson2013}
Myerson, R.~B.
\newblock \emph{Game Theory}.
\newblock Harvard University Press, 2013.

\bibitem[Omidshafiei et~al.(2017)Omidshafiei, Pazis, Amato, How, and
  Vian]{Omidshafiei17}
Omidshafiei, S., Pazis, J., Amato, C., How, J.~P., and Vian, J.
\newblock Deep decentralized multi-task multi-agent reinforcement learning
  under partial observability.
\newblock In \emph{Proc. ICML}, 2017.

\bibitem[Osband et~al.(2016)Osband, Blundell, Pritzel, and Roy]{Osband16}
Osband, I., Blundell, C., Pritzel, A., and Roy, B.~V.
\newblock Deep exploration via bootstrapped dqn.
\newblock In \emph{Proc. NeurIPS}, 2016.

\bibitem[Ostrovski et~al.(2017)Ostrovski, Bellemare, van~den Oord, and
  Munos]{Ostrovski17}
Ostrovski, G., Bellemare, M.~G., van~den Oord, A., and Munos, R.
\newblock Count-based exploration with neural density models.
\newblock In \emph{Proc. ICML}, 2017.

\bibitem[Parisotto et~al.(2019)Parisotto, Ghosh, Yalamanchi, Chinnaobireddy,
  Wu, and Salakhutdinov]{cmrl}
Parisotto, E., Ghosh, S., Yalamanchi, S.~B., Chinnaobireddy, V., Wu, Y., and
  Salakhutdinov, R.
\newblock Concurrent meta reinforcement learning.
\newblock In \emph{arXiv.}, 2019.

\bibitem[Pathak et~al.(2017)Pathak, Agrawal, Efros, and Darrell]{Pathak17}
Pathak, D., Agrawal, P., Efros, A.~A., and Darrell, T.
\newblock Curiosity-driven exploration by self-supervised prediction.
\newblock In \emph{Proc. ICML}, 2017.

\bibitem[Raileanu et~al.(2018)Raileanu, Denton, Szlam, and Fergus]{Raileanu18}
Raileanu, R., Denton, E., Szlam, A., and Fergus, R.
\newblock Modeling others using oneself in multi-agent reinforcement learning.
\newblock In \emph{Proc. ICML}, 2018.

\bibitem[Rangwala \& Williams(2020)Rangwala and Williams]{Rangwala20}
Rangwala, M. and Williams, R.
\newblock Learning multi-agent communication through structured attentive
  reasoning.
\newblock In \emph{Proc. NeurIPS}, 2020.

\bibitem[Rashid et~al.(2018)Rashid, Samvelyan, de~Witt, Farquhar, Foerster, and
  Whiteson]{qmix}
Rashid, T., Samvelyan, M., de~Witt, C.~S., Farquhar, G., Foerster, J., and
  Whiteson, S.
\newblock {QMIX:} monotonic value function factorisation for deep multi-agent
  reinforcement learning.
\newblock In \emph{Proc. ICML}, 2018.

\bibitem[Rashid et~al.(2020{\natexlab{a}})Rashid, Farquhar, Peng, and
  Whiteson]{wqmix}
Rashid, T., Farquhar, G., Peng, B., and Whiteson, S.
\newblock Weighted qmix: Expanding monotonic value function factorisation for
  deep multi-agent reinforcement learning.
\newblock In \emph{Proc. NeurIPS}, 2020{\natexlab{a}}.

\bibitem[Rashid et~al.(2020{\natexlab{b}})Rashid, Peng, Böhmer, and
  Whiteson]{OptimisticRashid20}
Rashid, T., Peng, B., Böhmer, W., and Whiteson, S.
\newblock Optimistic exploration even with a pessimistic initialisation.
\newblock In \emph{Proc. ICLR}, 2020{\natexlab{b}}.

\bibitem[Ronen I.~Brafman(2002)]{rmax}
Ronen I.~Brafman, M.~T.
\newblock R-max - a general polynomial time algorithm for near-optimal
  reinforcement learning.
\newblock In \emph{JMLR}, 2002.

\bibitem[Samvelyan et~al.(2019)Samvelyan, Rashid, de~Witt, Farquhar, Nardelli,
  Rudner, Hung, Torr, Foerster, and Whiteson]{smac}
Samvelyan, M., Rashid, T., de~Witt, C.~S., Farquhar, G., Nardelli, N., Rudner,
  T. G.~J., Hung, C.-M., Torr, P. H.~S., Foerster, J., and Whiteson, S.
\newblock The starcraft multi-agent challenge.
\newblock In \emph{arxiv.}, 2019.

\bibitem[Shen et~al.(2020)Shen, Zhao, Zhang, and Yu]{ModelShen20}
Shen, J., Zhao, H., Zhang, W., and Yu, Y.
\newblock Model-based policy optimization with unsupervised model adaptation.
\newblock In \emph{Proc. NeurIPS}, 2020.

\bibitem[Stadie et~al.(2016)Stadie, Levine, and Abbeel]{Stadie06}
Stadie, B.~C., Levine, S., and Abbeel, P.
\newblock Incentivizing exploration in reinforcement learning with deep
  predictive models.
\newblock In \emph{Proc. ICLR}, 2016.

\bibitem[Sunehag et~al.(2018)Sunehag, Lever, Gruslys, Czarnecki, Zambaldi,
  Jaderberg, Lanctot, Sonnerat, Leibo, Tuyls, and Graepel]{Sunehag18}
Sunehag, P., Lever, G., Gruslys, A., Czarnecki, W.~M., Zambaldi, V., Jaderberg,
  M., Lanctot, M., Sonnerat, N., Leibo, J.~Z., Tuyls, K., and Graepel, T.
\newblock Value-decomposition networks for cooperative multi-agent learning
  based on team reward.
\newblock In \emph{Proc. AAMAS}, 2018.

\bibitem[Sutton \& Barto(2018)Sutton and Barto]{SuttonRL}
Sutton, R.~S. and Barto, A.~G.
\newblock \emph{Reinforcement Learning: An Introduction}.
\newblock The MIT Press, 2018.

\bibitem[Swamy et~al.(2020)Swamy, Reddy, Levine, and Dragan]{Swamy20}
Swamy, G., Reddy, S., Levine, S., and Dragan, A.~D.
\newblock Scaled autonomy: Enabling human operators to control robot fleets.
\newblock In \emph{Proc. ICRA}, 2020.

\bibitem[Taiga et~al.(2020)Taiga, Fedus, Machado, Courville, and
  Bellemare]{aiga20}
Taiga, A.~A., Fedus, W., Machado, M.~C., Courville, A., and Bellemare, M.~G.
\newblock On bonus based exploration methods in the arcade learning
  environment.
\newblock In \emph{Proc. ICLR}, 2020.

\bibitem[Tampuu et~al.(2015)Tampuu, Matiisen, Kodelja, Kuzovkin, Korjus, Aru,
  Aru, and Vicente]{Tampuu15}
Tampuu, A., Matiisen, T., Kodelja, D., Kuzovkin, I., Korjus, K., Aru, J., Aru,
  J., and Vicente, R.
\newblock Multiagent cooperation and competition with deep reinforcement
  learning.
\newblock In \emph{arxiv.}, 2015.

\bibitem[Tan(1993)]{Tan93}
Tan, M.
\newblock Multiagent reinforcement learning independent vs cooperative agents.
\newblock In \emph{Proc. ICML}, 1993.

\bibitem[Tang et~al.(2017)Tang, Houthooft, Foote, Stooke, Chen, Duan, Schulman,
  Turck, and Abbeel]{Tang17}
Tang, H., Houthooft, R., Foote, D., Stooke, A., Chen, X., Duan, Y., Schulman,
  J., Turck, F.~D., and Abbeel, P.
\newblock Exploration: A study of count-based exploration for deep
  reinforcement learning.
\newblock In \emph{Proc. NeurIPS}, 2017.

\bibitem[Wang et~al.(2020)Wang, Wang, Wu, and Zhang]{eiti}
Wang, T., Wang, J., Wu, Y., and Zhang, C.
\newblock Influence-based multi-agent exploration.
\newblock In \emph{Proc. ICLR}, 2020.

\bibitem[Wu et~al.(2017)Wu, Mansimov, Liao, Grosse, and Ba]{acktr}
Wu, Y., Mansimov, E., Liao, S., Grosse, R., and Ba, J.
\newblock Scalable trust-region method for deep reinforcement learning using
  {Kronecker}-factored approximation.
\newblock In \emph{Proc. NeurIPS}, 2017.

\bibitem[Xu et~al.(2018)Xu, Liu, Zhao, and Peng]{Xu18}
Xu, T., Liu, Q., Zhao, L., and Peng, J.
\newblock Learning to explore via meta-policy gradient.
\newblock In \emph{Proc. ICML}, 2018.

\bibitem[Yang et~al.(2018)Yang, Luo, Li, Zhou, Zhang, and Wang]{Yang18}
Yang, Y., Luo, R., Li, M., Zhou, M., Zhang, W., and Wang, J.
\newblock Mean field multi-agent reinforcement learning.
\newblock In \emph{Proc. ICML}, 2018.

\bibitem[Zhang et~al.(2020)Zhang, Zhang, and Lin]{SuccinctZhang20}
Zhang, S.~Q., Zhang, Q., and Lin, J.
\newblock Succinct and robust multi-agent communication with temporal message
  control.
\newblock In \emph{Proc. NeurIPS}, 2020.

\bibitem[Zhou et~al.(2020)Zhou, Liu, Sui, Li, and Chung]{Zhou20}
Zhou, M., Liu, Z., Sui, P., Li, Y., and Chung, Y.~Y.
\newblock Learning implicit credit assignment for cooperative multi-agent
  reinforcement learning.
\newblock In \emph{Proc. NeurIPS}, 2020.

\end{thebibliography}
